\documentclass[letterpaper]{article} 
\usepackage{aaai2026}  
\usepackage{times}  
\usepackage{helvet}  
\usepackage{courier}  
\usepackage[hyphens]{url}  
\usepackage{graphicx} 
\usepackage{natbib}  
\usepackage{caption} 
\usepackage{url}            
\usepackage{booktabs}       
\usepackage{amsfonts}       
\usepackage{nicefrac}       
\usepackage{microtype}      
\usepackage{xcolor}         
\usepackage{colortbl}
\usepackage{algorithm}
\usepackage{algpseudocode}
\usepackage{amsmath}
\usepackage{multirow,amsthm}
\usepackage{cleveref}
\usepackage{subcaption}
\usepackage{bm}
\usepackage{graphicx}
\usepackage{enumitem}
\newtheorem{theorem}{Theorem}
\newtheorem{lemma}[theorem]{Lemma}
\newtheorem{corollary}[theorem]{Corollary}
\newtheorem{remark}[theorem]{Remark}
\crefname{figure}{Figure}{Figures}  
\crefname{table}{Table}{Tables} 
\crefname{equation}{Eq.}{Eqs.}
\crefname{section}{Section}{Section}
\usepackage{newfloat}
\usepackage{listings}
\DeclareCaptionStyle{ruled}{labelfont=normalfont,labelsep=colon,strut=off} 
\floatstyle{ruled}
\newfloat{listing}{tb}{lst}{}
\floatname{listing}{Listing}
\title{Generalizing Vision-Language Models with Dedicated Prompt Guidance}
\author{
    Xinyao Li\textsuperscript{\rm 1},
    Yinjie Min\textsuperscript{\rm 2},
    Hongbo Chen\textsuperscript{\rm 1},
    Zhekai Du\textsuperscript{\rm 1},
    Fengling Li\textsuperscript{\rm 3},
    Jingjing Li\textsuperscript{\rm 1}\thanks{Corresponding author.}
}
\affiliations{
    \textsuperscript{\rm 1}School of Computer Science and Engineering, University of Electronic Science and Technology of China\\
    \textsuperscript{\rm 2}School of Statistics and Data Science, Nankai University\\
    \textsuperscript{\rm 3}University of Technology Sydney\\
    xinyao326@outlook.com, nk.yjmin@gmail.com, jayus8643@gmail.com, zhekaid@uestc.edu.cn, fenglingli2023@gmail.com, lijin117@yeah.net

}

\usepackage{bibentry}

\begin{document}

\maketitle

\begin{abstract}
    Fine-tuning large pretrained vision-language models (VLMs) has emerged as a prevalent paradigm for downstream adaptation, yet it faces a critical trade-off between domain specificity and domain generalization (DG) ability. 
    Current methods typically fine-tune a universal model on the entire dataset, which potentially compromises the ability to generalize to unseen domains. To fill this gap, we provide a theoretical understanding of the generalization ability for VLM fine-tuning, which reveals that training multiple parameter-efficient expert models on partitioned source domains leads to better generalization  than fine-tuning a universal model. 
    Inspired by this finding, we propose a two-step domain-expert-\underline{Gui}ded DG (GuiDG) framework. GuiDG first employs prompt tuning to obtain source \textit{domain experts}, then introduces a Cross-Modal Attention module to guide the fine-tuning of the vision encoder via adaptive expert integration. 
    To better evaluate few-shot DG, we construct ImageNet-DG  from ImageNet and its variants. Extensive experiments on standard DG benchmarks and ImageNet-DG demonstrate that GuiDG  improves upon state-of-the-art fine-tuning methods while maintaining efficiency.  
\end{abstract} 
\begin{links}
    \link{Code}{https://github.com/TL-UESTC/GuiDG}
\end{links}

\section{Introduction}
\label{sec:intro}

Domain Generalization (DG) \cite{zhou2022domain} aims to learn general knowledge from multiple source domains that is applicable to unseen target distributions. While traditional approaches focus on extracting domain-invariant features \cite{li2018domain}, the emergence of general-purpose vision-language models (VLMs) like CLIP \cite{radford2021learning} has fundamentally changed this landscape. Thanks to extensive pretraining, these models demonstrate remarkable zero-shot generalization capability to novel objects and domains.

However, adapting CLIP to downstream tasks presents a challenge: \textbf{endowing the model with domain-specific knowledge while preserving its zero-shot generalization ability} \cite{wiseft,lai2023padclip}. Current methods tackle this by balancing specialization and generalization during fine-tuning. One line of work employs weight ensemble techniques, combining pretrained and fine-tuned model weights \cite{wiseft} to mitigate over-fitting. Another explores careful training paradigms to prevent catastrophic forgetting of pretrained knowledge \cite{lai2023padclip}.

\begin{figure}[t]
  \centering
  \includegraphics[width=0.96\linewidth]{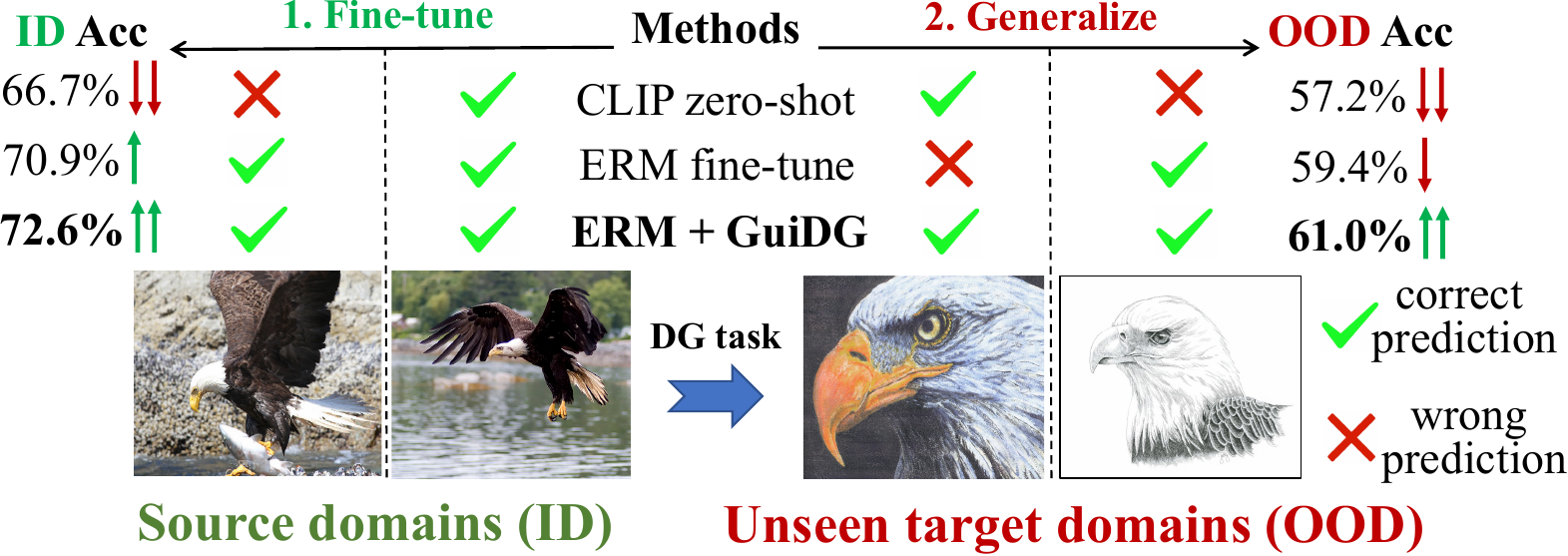}
  \caption{Illustration of the specialization - generalization balance. ERM fine-tuning fits to source knowledge at the cost of generalization ability,  while our GuiDG achieves consistent improvements on both seen and unseen domains.}
  \label{fig1}
\end{figure}

Despite their effectiveness, these approaches are limited by their design of sharing a universal model between source and target data. 
Such claim is supported by real-world examples (sampled from the ImageNet) in \cref{fig1}, where the model is overly fitted to limited source data but lacks adaptability to unknown distributions. A holistic source model cannot guarantee consistent performance across potential target domains. The source fine-tuning may even harm the zero-shot ability in pretrained CLIP. To address such challenge, we first derive a novel upper bound for DG risks that reveals two key insights. (1) Training with partitioned source data and reduced hypothesis space achieves lower generalization risks in fine-tuning. This motivates us to train multiple parameter-efficient prompts to serve as source \textit{domain experts} rather than fine-tuning a universal model. (2) An ensemble of dedicated source models provides more robust generalization than a universal model. Proper combination of experts can dynamically handle  unseen target distributions.

Based on the theoretical insights, we design a two-step framework termed domain-expert-Guided DG (GuiDG). In Step 1, we leverage prompt tuning \cite{coop} to learn domain experts dedicated to model each source domain. Prompt tuning updates less than 1\% of the total parameters in VLMs, providing a natural way to reduce the hypothesis space while capturing domain-specific knowledge. 
As shown in the left of \cref{fig1}, GuiDG discriminates better on source data than naive fine-tuning, which aligns with the benefits of reduced hypothesis space. Step 2 focuses on the combination of the domain experts. We introduce a lightweight Cross-Modal Attention (CMAttn) module that generates weights to determine the contribution of experts. By assigning larger weights to more compatible experts, CMAttn guides the vision encoder to learn better representations, which in turn enables CMAttn to learn better weighting strategies.

Only domain experts are learnable in Step 1. In Step 2, the experts are frozen while the vision encoder and CMAttn module are jointly optimized. Such design maintains computation efficiency while promoting cross-modal information exchange. As illustrated in the right part of \cref{fig1}, GuiDG generalizes well to unseen target domain while retaining zero-shot ability in CLIP.
The contributions of this work include:

\begin{itemize}[leftmargin=10pt]
  \item We conduct theoretical analysis for fine-tuning VLM and derive a novel upper bound for generalization risks. We reveal that, contrary to end-to-end fine-tuning, properly trained and aggregated domain-specific models can achieve better generalization than a single universal model.

  \item Guided by our theoretical findings, we propose GuiDG, a two-step framework that first learns parameter-efficient domain experts on partitioned source data, then employs a Cross-Modal Attention module to adaptively integrate these experts during VLM fine-tuning.

  \item We develop ImageNet-DG, a DG benchmark derived from ImageNet and its variants to evaluate few-shot DG. Extensive experiments demonstrate the consistent performance gains and parameter efficiency of GuiDG.
\end{itemize}

\section{Related Work}
\label{sec:related}

\textbf{Vision-Language Models (VLMs)} \cite{radford2021learning} are derived from web-scale image-text pairs with contrastive learning. VLMs generally feature a vision and text encoder to handle image and text inputs. By comparing vision and text representations, VLMs can make robust zero-shot inference \cite{li2025pataug}. Efforts have been made to adapt VLMs to downstream applications. Prompt-tuning methods \cite{coop,khattak2023maple} learn prompt embeddings to achieve parameter-efficient adaptation in a few-shot style. As a more practical scenario, some methods propose to distill the pretrained VLM to a smaller model for client-side deployment or fine-tuning \cite{addepalli2024leveraging,li2024promptkd}. VLMs' strong zero-shot ability has attracted researches on transfer learning tasks \cite{li2025generalizing}. Some works choose to integrate domain information in prompts \cite{ge2023domain,du2024domain}, while others refine the representation space to match target distribution \cite{11134143,li2024split}.

\textbf{Domain Generalization (DG)} aims to learn general knowledge from multiple source domains generalizes to unseen  domains. Adversarial-based methods \cite{li2018deep,deng2020representation} extract domain-invariant features from source domains via a min-max game between a feature extractor and domain discriminator. Augment-based methods refine source images to improve model generalization ability \cite{islam2024genmix,zhao2024style}. Zhou \textit{et al.} \cite{zhou2021domain,zhou2024mixstyle} mix-up images to synthesize novel domains that enhance model generalization. 
Inspired by meta-learning, some methods try to close domain shift by splitting source data into meta-train and meta-test subsets \cite{li2018learning,khoee2024domain}. With recent advances in VLMs, Chen \textit{et al.} \cite{chen2024practicaldg} propose to solve hybrid DG tasks with perturbation distillation of VLMs. Cheng \textit{et al.} \cite{cheng2024disentangled} utilize pretrained large language models to disentangle text prompts of VLMs for domain-agnostic visual features. Addepalli \textit{et al.} \cite{addepalli2024leveraging} solve white- and black-box DG settings by aligning vision and text representations before distillation. 
While there are attempts on ensemble-based DG~\cite{zhong2022meta,bai2024soft}, we are the first to theoretically support the benefits of such design, and propose a grounded GuiDG framework to reveal the generalization abilities of properly integrated domain-specific knowledge.

\textbf{Robust Fine-Tuning} focuses on preserving the generalizability of pretrained models while incorporating task-specific knowledge. For VLMs, prompt-tuning-based methods preserve base knowledge and learn new information \cite{zhang2024dept}, or learn instance-conditioned prompts to prevent over-fitting \cite{cocoop}. Unsupervised fine-tuning methods \cite{ueo,tanwisuth2023pouf} are based on zero-shot inference results and apply regularization terms to prevent forgetting. On the fine-tuning of large backbone networks, weight ensemble \cite{wiseft,swad} methods observe that mixing-up several weights in model optimization trajectory improves generalization performance. Adjusting learning rate also proves critical to preventing catastrophic forgetting \cite{wiseft,lai2023padclip}.  MIRO \cite{miro} proposes to constraint newly learned feature representations by oracle representations. 
This paper leverages domain specifics for generalization, which is compatible with existing fine-tuning methods.

\section{Method}
\label{sec:method}

\begin{figure*}[t]
    \centering
    \includegraphics[width=0.85\textwidth]{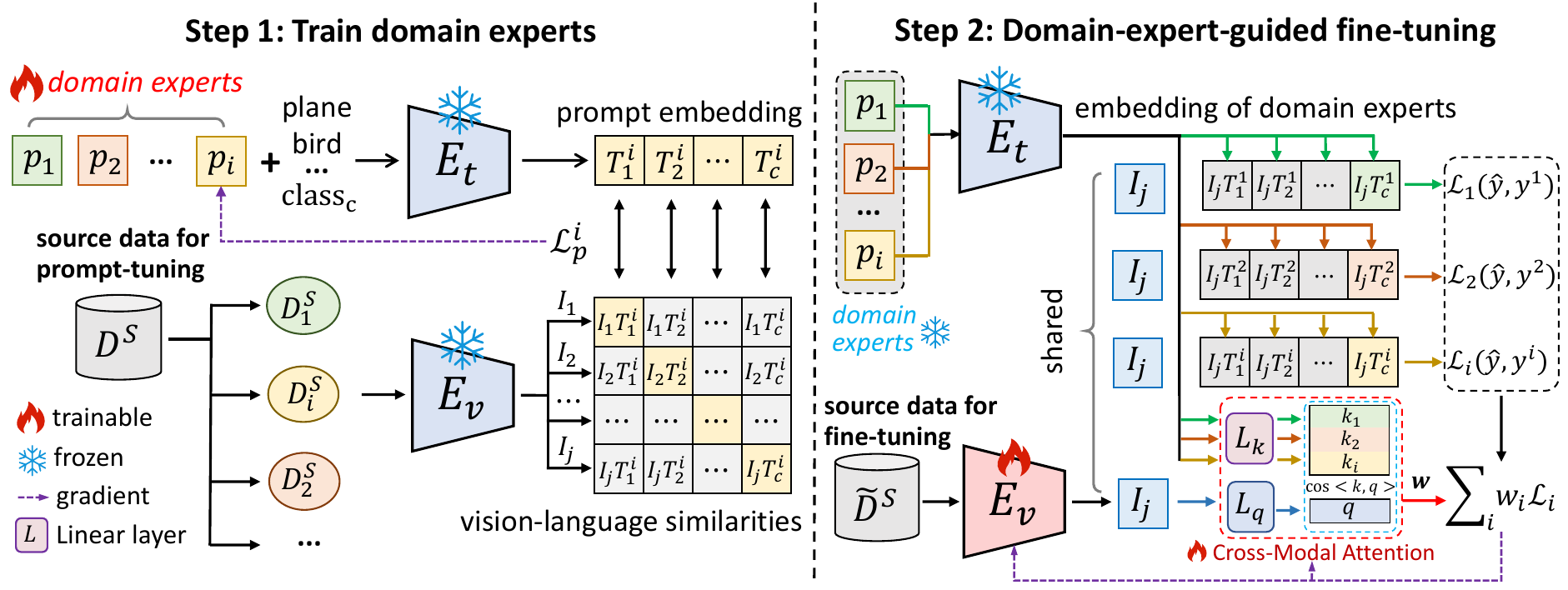}
    \caption{The two-step GuiDG framework. In Step 1, we split source  data  according to their domain characteristics. On each domain, a domain expert is learned with off-the-shelf prompt tuning methods. In Step 2, all domain experts are frozen. A Cross-Modal Attention (CMAttn) module decides ensemble weights from vision and text representations. These weights aggregate the knowledge in domain experts to guide the fine-tuning of the vision encoder, and assemble predictions for inference.}
    \label{fig2}
\end{figure*}

\subsection{Preliminaries}
\label{sec:preliminaries}
\textbf{Problem definition.} This work investigates domain generalizable fine-tuning for $C$-class classification. We have labeled source data partitioned into source domains $D^S=\{D^S_i\}_{i=1}^d$, where $d$ is the number of source domains and $D^S_i=\{(x^i_j,y^i_j)\}_{j=1}^{n^S_i}$ is the $i_{\mathrm{th}}$ source domain with $n^S_i$ labeled samples. The unseen target domain is denoted as $D^T=\{x^t_i\}_{i=1}^{n^t}$. The goal is to train a function $f^S$ on $D^S$ that minimizes prediction error on any unseen $D^T$. 

\textbf{Preliminaries on CLIP.} We investigate the generalizable fine-tuning of CLIP \cite{radford2021learning}. CLIP features a vision encoder $E_v$ that extract vision representations from input images $x$: $I=E_v(x)$. For each class, one can use a general description $t_c$, e.g., \texttt{ A photo of a [CLASS$_c$]}, for zero-shot classification task, where \texttt{CLASS}$_c$ is the category name of the $c_{\mathrm{th}}$ possible class. The text encoder $E_t$  takes the sentence $t_c$ as the input then generates text representations by $T_c=E_t(t_c)$. By computing cosine similarities ($\cos \left\langle , \right\rangle$) between the vision representation of image $x$ and text representations $\{T_c\}_{c=1}^{C}$ of all classes, we can obtain the probability that $x$ belongs to class $c$ by:
\begin{align}
    P(y=c \mid x)=\frac{\mathrm{exp}(\mathrm{cos} \left\langle I, T_c \right\rangle/\tau)}{\sum_{i=1}^{C}\mathrm{exp}(\mathrm{cos} \left\langle I, T_i \right\rangle /\tau)},
\label{zeroshot}
\end{align}
where $\tau$ is the temperature hyperparameter.

\subsection{Theoretical Formulation}
\label{sec:theory}
Assume that each source domain follows a distribution $P_i=P_{X,i}\times P_{Y\mid X}$, where $P_{X,i}$ represents the marginal distribution of features and $P_{Y\mid X}$ denotes the conditional distribution of labels given features.
The entire source domains follow a mixture distribution $P=P_X\times P_{Y\mid X}$, where the marginal feature distribution $P_X$ is defined as
$P_X=\sum_{i=1}^d \pi_i P_{X,i}$.
Similarly, the target domain follows distribution $P^\prime=P_X^\prime\times P_{Y\mid X}$ with
$P_X^\prime=\sum_{i=1}^d \pi_i^\prime P_{X,i}$.
Here $\pi_i,\pi_i^\prime\geq 0$ holds for all $i=1,\ldots,d$, and satisfies
$\sum_{i=1}^d \pi_i = \sum_{i=1}^d \pi_i^\prime = 1$.
Let $n=\sum_{i=1}^d n_i^S$ be the total number of source samples. We assume that $n_i^S=\pi_in$ for all $i=1,\ldots,d$ \cite{albuquerque2019generalizing}.
For any hypothesis $h \in \mathcal{H}$ and distribution $P_i$, we define its risk as $\mathcal{E}_i(h)=\mathbb{E}_{x,y\sim P_i}\mathcal{L}(y,h(x))$, where $\mathcal{L}(\cdot,\cdot)$ is a loss function bounded by $c_L$. Our goal is to find $h \in \mathcal{H}$ that minimizes the risk $\mathcal{E}^\prime(h)=\mathbb{E}_{x,y\sim P^\prime} \mathcal{L}(y,h(x))$ on the target distribution $P^\prime$. 
The classical approach seeks to train a universal predictor $\widehat{f}$ by empirical risk minimization (ERM) \cite{gulrajani2020search} across all domains:
\begin{align}
    \widehat{f} = \underset{f\in\mathcal{H}}{\arg\min} \, \sum_{i=1}^{d} \pi_i\widehat{\mathcal{E}}_i(f),
\label{eq0}
\end{align}
where $\widehat{\mathcal{E}}_i(h)=\frac{1}{n_i^S} \sum_{j=1}^{n_i^S} \mathcal{L}(y_j^i,h(x_j^i))$ denotes the empirical risk on domain $i$.
Instead, we propose a two-stage approach. First, for each domain $i$ with hypothesis space $\mathcal{H}_i$, we find a domain-specific predictor 
$\widehat{f}_i=\arg \min_{f \in \mathcal{H}_i} \, \widehat{\mathcal{E}}_i(f)$. 
Then, we aggregate these predictors $\{\widehat{f}_i\}_{i=1}^d$ using algorithm $\mathcal{A}$ and an additional independent dataset $\widetilde{D}^S=\{\widetilde{D}_i^S\}_{i=1}^d$, where $\widetilde{D}_i^S=\{(\tilde{x}_j^i,\tilde{y}_j^i)\}_{j=1}^{\widetilde{n}_i^S}$ is also independently drawn from $P_i$. Let $m=\sum_{i=1}^d\widetilde{n}_i^S$ denote the total number of samples in $\widetilde{D}^S$ with $\widetilde{n}_i^S=\pi_i m$. The aggregated predictor $\tilde{f}=\mathcal{A} (\widehat{f}_1,\ldots,\widehat{f}_d;\widetilde{D}^S)$ belongs to hypothesis space $\widetilde{\mathcal{H}}$. For comparison, we redefine $\widehat{f}$ as the minimizer in $\mathcal{H}$ of the empirical loss on $D^S\cup\widetilde{D}^S$.

\begin{theorem}\label{theo: general}
    Assume hypothesis space $\mathcal{H}$, $\widetilde{\mathcal{H}}$ and $\mathcal{H}_i$ have VC-dimension $d_0$, $\widetilde{d}$ and $d_i$ respectively. There exists constant $C>0$, such that for any $\delta\in(0,1)$ with probability at least $1-3d\delta$, the following inequality hold:
    \begin{align}
        \mathcal{E}^\prime(\widetilde{f})-\overset{d}{\underset{i=1}\sum}&\pi_i^\prime\widehat{\mathcal{E}}_i(\widehat{f}_i)\leq\left( \overset{d}{\underset{i=1}{\sum}}\pi_i^\prime/\sqrt{\pi_i} \right)\sqrt{\dfrac{c_L\log(1/\delta)}{2m}}\notag\\
        &+C\sqrt{\dfrac{\widetilde{d}\log(m)+\log(1/\delta)}{m}}\notag\\
        &+C\overset{d}{\underset{i=1}\sum}\pi_i^\prime\sqrt{\dfrac{d_i\log(n_i^S)+\log(1/\delta)}{n_i^S}}\,,\label{form: boundforAgg}
    \end{align}
    and denote $N=n+m$, with probability at least $1-d\delta$ the following inequality hold:
    \begin{align}
        \mathcal{E}^\prime(\widehat{f})-\overset{d}{\underset{i=1}\sum}\pi_i^\prime\widehat{\mathcal{E}}_i(\widehat{f})\leq C\sqrt{\dfrac{d_0\log(N)+\log(1/\delta)}{N}}\,.\label{form: boundforClassical}
    \end{align}
\end{theorem}

Denote the right-hand side of \cref{form: boundforAgg} and \cref{form: boundforClassical} as $\mathrm{Upp}(\mathcal{E}^\prime(\widetilde{f}),\delta)$ and $\mathrm{Upp}(\mathcal{E}^\prime(\widehat{f}),\delta)$, we have corollary:
\begin{corollary}\label{corr: general}
    Assume $m=n$, $c_\pi=\sum_{i=1}^{d} \pi_i^\prime/\sqrt{\pi_i}$ and $\sum_{i=1}^{d} \pi_i^\prime\sqrt{2d_i}/\sqrt{\pi_i}\leq c(\delta)\sqrt{d_0}$, where for specified $\delta\in(0,1)$, $c(\delta)=\underset{1\leq i\leq d}{\inf}\sqrt{\dfrac{\log(2n)+(1/d_0)\log(1/\delta)}{\log(n_i^S)+(1/d_i)\log(3/\delta)}}$. We have
    \begin{equation}
        \mathrm{Upp}(\mathcal{E}^\prime(\widetilde{f}),\delta/3)\leq\mathrm{Upp}(\mathcal{E}^\prime(\widehat{f}),\delta)+\varepsilon\,,\label{formula: aggclaineq}
    \end{equation}
    where
    \begin{equation*}
        \varepsilon=c_\pi\sqrt{\dfrac{c_L\log(3/\delta)}{N}}+C\sqrt{\dfrac{2\widetilde{d}\log(N/2)+2\log(3/\delta)}{N}}\,.
    \end{equation*}
\end{corollary}

\begin{remark}\label{remark1}
    When $\mathcal{H}$ is a parameterized neural network space, assume the number of parameters as $n(\mathcal{H})$. The VC-dimension of $\mathcal{H}$ is approximately $n(\mathcal{H})\log\{n(\mathcal{H})\}$ \cite{bartlett2003vapnik}. When $\widetilde{d}\ll d_0$, which is equivalent to $n(\widetilde{\mathcal{H}})\ll n(\mathcal{H})$, $\varepsilon$ in inequality (\ref{formula: aggclaineq}) is a very small item compared with $\mathrm{Upp}(\mathcal{E}^\prime(\widehat{f}),\delta)$ as $d_0$ is usually large. Corollary \ref{corr: general} suggests that we should seek ways to make $\sum_{i=1}^d \pi_i^\prime\sqrt{2n(\mathcal{H}_i)\log\{n(\mathcal{H}_i)\}}/\sqrt{\pi_i} < c(\delta)\sqrt{n(\mathcal{H})\log\{n(\mathcal{H})\}}$, thereby ensuring the upper bound of ensemble risk is much smaller than that of a universal model, i.e., $\mathrm{Upp}(\mathcal{E}^\prime(\widetilde{f}),\delta/3) < \mathrm{Upp}(\mathcal{E}^\prime(\widehat{f}),\delta)$.
\end{remark}
With carefully constructed $\mathcal{H}_i$ that satisfies condition in Corollary \ref{corr: general}, when $\widehat{\mathcal{E}}_i(\widehat{f}_i)\leq \widehat{\mathcal{E}}_i(\widehat{f})$ (each $\widehat{f}_i$ is trained to minimize empirical risk of its domain with a much smaller hypothesis space, leading to lower empirical risk compared to the universal model $\widehat{f}$ that needs to compromise across all domains), $\mathcal{E}^\prime(\widetilde{f})$ has a tighter upper bound than $\mathcal{E}^\prime(\widehat{f})$ on $P^\prime$. The proofs on Theorem \ref{theo: general} and Corollary \ref{corr: general}, as well as an illustrative toy example on Remark \ref{remark1} are in Appendix.

\subsection{Learning Domain Experts}
Current  methods \cite{wiseft} fully fine-tune a universal model on all source data. However, Remark \ref{remark1} indicates that  an ensemble of parameter-efficient domain-specific models brings better generalization ability. Motivated by such finding, we propose to learn \textbf{dedicated} function $\widehat{f}_i$ from each carefully designed $\mathcal{H}_i$ to ensure $\sum_{i=1}^{d} \pi_i^\prime\sqrt{2d_i}/\sqrt{\pi_i}\leq c(\delta)\sqrt{d_0}$. More details are in Appendix. The learned experts incorporate more specific domain knowledge into the default text description, e.g., \textit{A \underline{[real/art/...]} photo of a [CLASS].}

We learn the domain experts in a few-shot style \cite{coop}, while any  prompt-tuning  method is viable. On domain $D^S_i$, we construct the learnable prompt for class $j$ as: 
\begin{align}
    t^i_j=[p_{i1}][p_{i2}]...[p_{im}][\mathrm{CLASS_j}],
\end{align}
where $\mathbf{p_i}=[p_{i1}][p_{i2}]...[p_{im}]$ is domain expert, $m$ is length of expert, and [CLASS$_\mathrm{j}$] is embedding of class $j$. The text embedding of class $j$ in domain $i$ is  obtained: $T^i_j=E_t(t^i_j)$. The probability of image $x^i$ belonging to class $j$ is obtained:
\begin{align}
    P_j(\hat{y}\mid x^i,t^i_{1:C})=\frac{\mathrm{exp}(\mathrm{cos} \left\langle E_v(x^i), T^i_j \right\rangle /\tau)}{\sum_{c=1}^{C}\mathrm{exp}(\mathrm{cos}\left\langle E_v(x^i), T^i_c \right\rangle /\tau)}.
\label{eq3}
\end{align}
Standard cross-entropy loss is used to learn $\mathbf{p_i}$:
\begin{align}
    \mathcal{L}^i_p= - \sum_{j=1}^{n_i^S} \sum_{c=1}^{C} [y_j^i=c] \cdot \log P_c(\hat{y}\mid x^i_j,t^i_{1:C}),
\label{lp}
\end{align}
where $[\cdot]$ is indicator function. Both $E_v$ and $E_t$ are frozen during the optimization of \cref{lp}. The $i_{\text{th}}$ domain expert is obtained by optimizing  \cref{lp}: $\mathbf{p_i} = \underset{\mathbf{p_i}}{\arg\min}  \, \mathcal{L}^i_p. $

\subsection{Domain-Expert-Guided Fine-Tuning}
With trained domain experts $\mathbf{p_i}$ to serve as dedicated domain predictors $\widehat{f}_i$, we proceed to learn algorithm $\mathcal{A}$ that aggregates the domain knowledge in $\{\widehat{f}_i\}_{i=1}^d$ to guide the fine-tuning of CLIP. 
As shown in Step 2 of \cref{fig2}, we design a Cross-Modal Attention (CMAttn) module to approximate the aggregation algorithm $\mathcal{A}$.  CMAttn is composed of a linear layer $L_q$ that transforms vision features into query embeddings: $q(x)=L_q(E_v(x))$, and a linear layer $L_k$ to transform text features into key embeddings: $k_i=L_k(E_t(t_i))$. We then compute the normalized cosine similarities between a query embedding and key embeddings of all domain experts to obtain ensemble weights:
\begin{align}
    \mathbf{w}(x) = \mathrm{Softmax}(\mathrm{cos} \left\langle q(x), [k_1,k_2,\cdots,k_d] \right\rangle ),
\label{attn}
\end{align}
where $\mathbf{w}(x)=(w_1(x),w_2(x),...,w_d(x))$ is the weight vector, and $[\cdot]$ is the concatenation operation. CMAttn learns to assign larger weights to more compatible experts, and smaller weights to irrelevant experts. The training data in Step 2 include data from all available source domains, therefore CMAttn can learn various weighting scenarios that are generalizable to unseen target domains. 
For an image input, different classification results are obtained from each domain expert via \cref{eq3}. Instead of computing their weighted average before optimization, we propose to weight the training losses. We observe that such design brings more direct and effective optimization for both the image encoder and CMAttn module. As instructed in Theoretical Formulation, we train CMAttn on $\widetilde{D}^S$ with $m$ samples that are \textit{i.i.d} with $D^S$. Combining \cref{eq3}, the training loss for Step 2 is defined:
\begin{align}
    \mathcal{L}_{f} = \sum_{i=1}^{d} \sum_{j=1}^{\widetilde{n}^S_i} w_i(x_j^i)  \left( - \sum_{c=1}^{C} [y_j^i=c]  \log P_c(\hat{y}|x_j^i,t^i_{1:C}) \right).
\label{lf}
\end{align} 
During the optimization of \cref{lf}, only the CMAttn and the vision encoder are trainable.

\begin{table*}[!t]
    \centering
    \small
    \resizebox{0.95\linewidth}{!}{
    \begin{tabular}{l|ccccc|ccccccc}
    \hline
     & \multicolumn{5}{c}{OfficeHome} & \multicolumn{7}{|c}{DomainNet} \\
    \multirow{-2}{*}{Method} & Art & Clp & Prod & RW & Avg. & clp & inf & pnt & qdr & rel & skt & Avg. \\ \hline
    CLIP-zeroshot  & 82.9 & 67.8 & 89.0 & 89.8 & 82.4 & 70.1 & 46.4 & 61.7 & 13.7 & 82.9 & 62.6 & 56.2 \\ \hline
    \rowcolor[HTML]{D9D9D9} 
    \textbf{Full data} &  &  &  &  &  &  &  &  &  &  &  &  \\
    MIRO \cite{miro} & 83.6 & 75.7 & 89.7 & 90.2 & 84.8 & 79.7 & 43.5 & 67.4 & \textbf{24.6} & 79.2 & 68.4 & 60.5 \\
    WiSE-FT \cite{wiseft} & 85.2 & 76.2 & 92.9 & 91.0 & 86.3 & 76.8 & 49.5 & 69.4 & 20.1 & 81.7 & 67.2 & 60.8 \\
    VL2V-SD \cite{addepalli2024leveraging} & 87.3 & \textbf{78.6} & 92.0 & 91.7 & 87.4 & \textbf{80.0} & 49.0 & 71.1 & 23.3 & 82.1 & \textbf{71.4} & 62.8 \\
    ERM (WF)*  & 84.9 & 71.2 & 92.4 & 92.0 & 85.1$_{\pm0.2}$ & 74.5 & 49.5 & 69.6 & 16.1 & 84.5 & 66.7 & 60.2$_{\pm0.1}$ \\
    \rowcolor[HTML]{DDEBF7}
    ERM (WF) + GuiDG* & 85.9 & 71.7 & 92.6 & 92.3 & 85.6$_{\pm0.2}$  & 76.0 & 53.1 & 70.9 & 17.3 & 84.8 & 68.2 & \underline{61.7}$_{\pm0.2}$  \\
    UEO (WF)* \cite{ueo} & 85.6 & 72.8 & \textbf{93.2} & 92.5 & 86.0$_{\pm0.3}$ & 75.5 & 50.8 & 70.2 & 16.9 & 84.6 & 66.8 & 60.8$_{\pm0.1}$ \\
    \rowcolor[HTML]{DDEBF7}
    UEO (WF) + GuiDG* & 86.8 & 73.6 & 92.9 & 93.4 & \underline{86.7}$_{\pm0.3}$  & 76.5 & 52.9 & 71.1 & 17.8 & 85.0 & 68.9 & 62.0$_{\pm0.2}$  \\
    CLIPood* \cite{clipood} & 87.8 & 73.8 & 92.7 & 92.9 & 86.8$_{\pm0.2}$ & 77.6 & \textbf{54.6} & \textbf{72.7} & 20.8 & 85.2 & 69.7 & \textbf{63.4$_{\pm0.1}$} \\
    \rowcolor[HTML]{DDEBF7} 
    CLIPood + GuiDG* & \textbf{89.1} & 74.6 &  92.1 & \textbf{93.6} & \textbf{87.4$_{\pm0.2}$} & 77.7 & 54.3 & 72.5 & 20.4 & \textbf{85.3} & 69.9 & \textbf{63.4$_{\pm0.2}$} \\ \hline
    \rowcolor[HTML]{D9D9D9} 
    \textbf{16-shot} &  &  &  &  &  & \textbf{} &  &  &  &  &  &  \\
    ERM (WF)*  & 86.0 & 69.5 & 92.3 & 93.0 & 85.2$_{\pm0.3}$ & 73.9 & 49.8 & 68.3 & 16.3 & 84.6 & 65.6 & 59.8$_{\pm0.2}$ \\
    \rowcolor[HTML]{DDEBF7} 
    ERM (WF) + GuiDG* & 86.4 & 70.0 & 93.1 & 91.6 & 85.3$_{\pm0.3}$  & 74.2 & 51.6 & 68.5 & 15.0 & 84.9 & 66.8 & 60.2$_{\pm0.2}$  \\
    UEO (WF)* \cite{ueo} & 85.4 & 68.1 & 92.7 & 92.8 & 84.8$_{\pm0.3}$ & 74.1 & 52.9 & 68.3 & 14.9 & 85.1 & 66.9 & 60.4$_{\pm0.2}$ \\
    \rowcolor[HTML]{DDEBF7} 
    UEO (WF) + GuiDG* & 87.2 & 71.1 & 93.2 & 92.6 & \underline{86.0}$_{\pm0.2}$  & 75.4 & 52.9 & 69.9 & 17.6 & \textbf{85.4} & 68.3 & \underline{61.6}$_{\pm0.2}$  \\
    CLIPood* \cite{clipood} & 86.2 & 71.7 & 92.7 & 93.2 & 86.0$_{\pm0.3}$ & \textbf{77.3} & 53.0 & 70.9 & 19.3 & 85.1 & 68.2 & 62.3$_{\pm0.1}$ \\
    \rowcolor[HTML]{DDEBF7} 
    CLIPood + GuiDG* & \textbf{87.6} & \textbf{72.9} & \textbf{93.4} & \textbf{93.4} & \textbf{86.8$_{\pm0.4}$} & 76.8 & \textbf{53.7} & \textbf{71.7} & \textbf{20.3} & 85.2 & \textbf{68.9} & \textbf{62.8$_{\pm0.2}$} \\ \hline
    \rowcolor[HTML]{D9D9D9}  
    \textbf{8-shot} &  &  &  &  &  & \textbf{} &  &  &  &  &  & \\
    ERM (WF)*  & 82.9 & 65.3 & 90.5 & 92.2 & 82.7$_{\pm0.4}$ & 71.5 & 46.3 & 66.2 & 13.9 & 83.8 & 63.9 & 57.6$_{\pm0.2}$\\
    \rowcolor[HTML]{DDEBF7} 
    ERM (WF) + GuiDG* & 85.8 & 69.7 & \textbf{93.0} & 92.1 & \underline{85.2}$_{\pm0.3}$  & 73.9 & 51.5 & 68.4 & 14.9 & 85.0 & 66.2 & \underline{60.0}$_{\pm0.2}$  \\
    UEO (WF)* \cite{ueo} & 85.1 & 67.4 & 92.2 & 92.3 & 84.3$_{\pm0.2}$ & 74.3 & 51.6 & 68.7 & 15.1 & 84.9 & 66.5 & 60.2$_{\pm0.2}$ \\
    \rowcolor[HTML]{DDEBF7} 
    UEO (WF) + GuiDG* & 86.2 & 72.1 & 92.3 & 92.7 & 85.8$_{\pm0.3}$  & 76.1 & 52.7 & 70.2 & 17.9 & \textbf{85.3} & 67.9 & 61.7$_{\pm0.3}$  \\
    CLIPood* \cite{clipood} & \textbf{86.8} & 69.9 & 92.7 & \textbf{93.9} & 85.8$_{\pm0.3}$ & \textbf{76.9} & 51.7 & 70.5 & \textbf{18.9} & 84.7 & 68.0 & 61.8$_{\pm0.2}$ \\
    \rowcolor[HTML]{DDEBF7} 
    CLIPood + GuiDG* & 86.6 & \textbf{73.5} & 92.4 & 92.9 & \textbf{86.4$_{\pm0.3}$} & 76.1 & \textbf{53.1} & \textbf{71.0} & 18.8 & 85.2 & \textbf{68.5} & \textbf{62.1$_{\pm0.2}$} \\\hline
    \multicolumn{3}{l}{* Results based on our own runs.}
    \end{tabular}}
    \caption{DG results of GuiDG. Best results are in bold. Most significant improvements by incorporating GuiDG are underlined.}
    \label{tab1}
\end{table*}

\subsection{Train and Inference}
As shown in \cref{fig2}, GuiDG consists of two training steps. In Step 1, we train $d$ independent domain experts $\mathbf{p_1},\mathbf{p_2},\cdots,\mathbf{p_d}$ by optimizing \cref{lp}. In Step 2, CMAttn and the vision encoder are trained by minimizing \cref{lf}. Step 2 is compatible with existing regularization terms (detailed in Appendix) for robust fine-tuning, therefore we have:
\begin{align}
    \theta_{E_v}, \theta_{L_q}, \theta_{L_k} =  \underset{\theta_{E_v}, \theta_{L_q}, \theta_{L_k}}{\arg \, \min} \, \mathcal{L}_f + \alpha \mathcal{L}_r,
    \label{opt2}
\end{align}
where $\theta_{E_v}, \theta_{L_q}, \theta_{L_k}$ are parameters in the vision encoder and CMAttn, $\mathcal{L}_r$ is off-the-shelf regularization loss for fine-tuning, and $\alpha$ controls the regularization effects.

During inference, the VLM generates $d$ sets of logits from each domain expert given target data $x^t$. CMAttn assigns proper weights $w_i$ for each output. Assume hidden variable $\mathcal{I}(x^t)$ indicating the index of $\widehat{f}_i$ that best predicts $y^t$. For class $c$, by conditioning on $\{\widehat{f}_i\}_{i=1}^d$ we have
$\mathbb{P}\left( y^t=c\mid x^t \right) = \sum_{i=1}^{d} \mathbb{P}\left( \mathcal{I}(x^t)=i \right)\mathbb{P}\left( y^t=c\mid x^t,\mathcal{I}(x^t)=i \right).$
In our design, weight $w_i(x^t)$ is an estimator of $\mathbb{P}\left( \mathcal{I}(x^t)=i \right)$ and $\cos \left\langle E_v(x^t),T^i_c \right\rangle / \tau$ is an estimator approximately proportion to $\mathbb{P}\left( y^t=c\mid x^t,\mathcal{I}(x^t)=i \right)$.
Thus, the weighted average of outputs serve as the final inference results $\hat{y}^t$: 
\begin{align}
    \hat{y}^t =  \underset{c}{\arg \max} \, \sum_{i=1}^{d} w_i(x^t) \cdot \cos \left\langle E_v(x^t),T^i_c \right\rangle / \tau.
    \label{eq10}
\end{align}

\section{Experiments}
\subsection{Setup}
\label{sec:setup}
\textbf{Benchmark.} We conduct experiments on standard domain generalization benchmarks. All experiments are repeated 5 times with different seeds and means were reported. \textbf{OfficeHome} \cite{venkateswara2017deep} includes 4 domains with 65 categories of office items. \textbf{VLCS} \cite{torralba2011unbiased} includes 4 domains with 5 classes. \textbf{PACS} \cite{li2017deeper} provides 4 art-style domains with 7 classes. \textbf{DomainNet} \cite{peng2019moment} contains 0.6 million samples from 6 domains and 345 categories. \textbf{TerraIncognita (TI)} \cite{beery2018recognition} includes 10 classes of animal pictures taken in 4 different  locations. 
We notice  existing DG benchmarks include limited  classes and instances, hindering sufficient evaluation of modern models like CLIP. Therefore, we sample data from ImageNet \cite{deng2009imagenet} and its variants \cite{hendrycks2021natural,hendrycks2021many,wang2019learning,recht2019imagenet} to construct a new subset \textbf{ImageNet-DG}. We also experiment on single-source DG, where the model is fine-tuned on ImageNet and generalizes to  ImageNet variants.

\begin{table}[t]
    \centering
    \small
    \resizebox{\linewidth}{!}{
    \begin{tabular}{l|cccc}
    \hline
    Method & PACS & VLCS & TI & Avg. \\ \hline
    CLIP \cite{radford2021learning} & 96.2 & 81.8 & 33.8 & 70.6 \\
    MIRO \cite{miro} & 95.6 & 82.2 & 54.3 & 77.4 \\
    SWAD \cite{swad} & 91.4 & 79.1 & 42.9 & 71.1 \\
    WiSE-FT \cite{wiseft} & 97.3 & 82.9 & 54.5 & 78.2 \\
    RISE \cite{rise} & 93.3 & 80.6 & 49.6 & 74.5 \\
    VL2V-SD \cite{addepalli2024leveraging} & 96.7 & 83.3 & 58.5 & 79.5 \\
    DSPL \cite{cheng2024disentangled} & 97.5 & \textbf{86.4} & 57.1 & 80.3 \\
    ERM (WF)* & 96.7$_{\pm0.4}$ & 83.3$_{\pm0.2}$ & 51.0$_{\pm0.4}$ & 77.0 \\
    \rowcolor[HTML]{DDEBF7} 
    ERM (WF) + GuiDG*  & \textbf{\underline{97.6}}$_{\pm0.2}$ & 83.8$_{\pm0.2}$ & \underline{52.9}$_{\pm0.5}$ & 78.1 \\
    UEO (WF)*  \cite{ueo} & 96.9$_{\pm0.2}$ & 81.3$_{\pm0.3}$ & 51.5$_{\pm0.3}$ & 76.6 \\
    \rowcolor[HTML]{DDEBF7} 
    UEO (WF) + GuiDG* & \textbf{97.6$_{\pm0.2}$} & \underline{84.2}$_{\pm0.2}$ & 52.3$_{\pm0.4}$ & \underline{78.0} \\
    CLIPood*  \cite{clipood} & 97.3$_{\pm0.1}$ & 84.2$_{\pm0.2}$ & 60.0$_{\pm0.6}$ & 80.5 \\
    \rowcolor[HTML]{DDEBF7} 
    CLIPood +  GuiDG*  & 97.3$_{\pm0.1}$ & 84.3$_{\pm0.2}$ & \textbf{60.9$_{\pm0.4}$} & \textbf{80.8} \\ \hline
    \multicolumn{3}{l}{* Results based on our own runs.}
    \end{tabular}}
    \caption{DG results on PACS, VLCS and TI. Best results are in bold. Most significant improvements are underlined.}
    \label{tab2}
\end{table}

\begin{table*}[!t]
    \centering
    \small
    \centering
    \begin{tabular}{l|cccc|cccccc}
    \hline
     & \multicolumn{4}{c}{ImageNet-DG} & \multicolumn{6}{|c}{Single-source ImageNet} \\ 
     & \multicolumn{3}{c}{source = I, S, V2} & \multicolumn{1}{c}{} & \multicolumn{5}{|c}{source = I} & \multicolumn{1}{c}{} \\\cline{2-4} \cline{6-10}
    \multirow{-3}{*}{Method} & \multicolumn{1}{c}{A} & \multicolumn{1}{c}{I} & \multicolumn{1}{c}{R} & \multicolumn{1}{c}{\multirow{-2}{*}{Avg.}} & \multicolumn{1}{|c}{A} & \multicolumn{1}{c}{I} & \multicolumn{1}{c}{R} & \multicolumn{1}{c}{S} & \multicolumn{1}{c}{V2} & \multicolumn{1}{c}{\multirow{-2}{*}{Avg.}} \\ \hline
    CLIP-zeroshot \cite{radford2021learning} & 47.0 & 66.7 & 73.9 & 62.6 & 47.8 & 66.7 & 74.0 & 46.1 & 60.8 & 59.1 \\\hline
    \rowcolor[HTML]{D9D9D9} 
    \textbf{16-shot} &  &  &  & &  &  &  &  &  &  \\
    PromptSRC$^\dagger$ \cite{khattak2023self} & \textbf{51.0} & 72.1 & 79.1 & 67.4 & 50.9 & 71.3 & 77.8 & 49.6 & 64.4 & 62.8 \\
    Apex$^\dagger$ \cite{yang2023towards} & 50.6 & 72.5 & 78.7 & 67.3 & 50.7 & 72.0 & 76.8 & 48.5 & 64.7 & 62.5 \\
    ERM (WF)* \cite{gulrajani2020search} & 48.9 & 70.7 & 78.9 & 66.2$_{\pm0.1}$ & 49.1 & 70.9 & 75.8 & 48.4 & 64.1 & 61.7$_{\pm0.1}$ \\
    \rowcolor[HTML]{DDEBF7} 
    ERM (WF) +  GuiDG* & 50.2 & 72.8 & 80.8 & 67.9$_{\pm0.2}$ & \textbf{51.1} & 72.6 & 77.9 & 49.7 & 65.3 & 63.3$_{\pm0.1}$ \\
    UEO (WF)* \cite{ueo} & 50.5 & 69.8 & 77.2 & 65.8$_{\pm0.2}$ & 48.6 & 71.9 & 75.9 & 48.6 & 65.1 & 62.0$_{\pm0.1}$ \\
    \rowcolor[HTML]{DDEBF7} 
    UEO (WF) +   GuiDG* & 50.8 & 72.9 & 80.7 & \underline{68.1}$_{\pm0.2}$ & 51.0 & \textbf{73.5} & \textbf{78.1} & 50.1 & 66.2 & \textbf{\underline{63.8}$_{\pm0.2}$} \\
    CLIPood* \cite{clipood} & 49.8 & 71.8 & 81.1 & 67.6$_{\pm0.1}$ & 50.4 & 71.6 & 77.2 & 49.3 & 64.9 & 62.7$_{\pm0.1}$ \\
    \rowcolor[HTML]{DDEBF7} 
    CLIPood +   GuiDG* & 50.9 & \textbf{73.0} & \textbf{81.8} & \textbf{68.6$_{\pm0.1}$} & 49.6 & 73.4 & 77.7 & \textbf{50.4} & \textbf{66.3} & 63.5$_{\pm0.1}$ \\\hline
    \rowcolor[HTML]{D9D9D9} 
    \textbf{8-shot} &  &  &  &  &  &  &  &  &  &  \\
    ERM (WF)* \cite{gulrajani2020search} & 49.0 & 69.8 & 77.4 & 65.4$_{\pm0.1}$ & 47.9 & 70.3 & 75.9 & 48.2 & 63.3 & 61.1$_{\pm0.1}$ \\
    \rowcolor[HTML]{DDEBF7} 
    ERM (WF) +   GuiDG* & \textbf{50.7} & 71.3 & 79.4 & \underline{67.1}$_{\pm0.1}$ & \textbf{50.9} & 71.9 & \textbf{78.1} & 49.5 & 64.8 & 63.0$_{\pm0.1}$ \\
    UEO (WF)* \cite{ueo} & 50.5 & 70.0 & 77.0 & 65.8$_{\pm0.1}$ & 48.0 & 70.3 & 76.2 & 48.5 & 63.4 & 61.3$_{\pm0.2}$ \\
    \rowcolor[HTML]{DDEBF7} 
    UEO (WF) +   GuiDG* & 50.4 & 71.7 & 80.1 & 67.4$_{\pm0.1}$ & 50.5 & 72.8 & 78.0 & 49.8 & 65.6 & \textbf{63.3$_{\pm0.2}$} \\
    CLIPood* \cite{clipood} & 50.0 & 71.6 & 80.5 & 67.4$_{\pm0.2}$ & 45.5 & 71.7 & 75.2 & 47.9 & 64.8 & 61.0$_{\pm0.1}$ \\
    \rowcolor[HTML]{DDEBF7} 
    CLIPood +   GuiDG* & 50.3 & \textbf{71.8} & \textbf{81.2} & \textbf{67.8$_{\pm0.2}$} & 49.5 & \textbf{73.0} & 77.3 & \textbf{50.2} & \textbf{65.7} & \underline{63.1}$_{\pm0.1}$ \\ \hline
    \multicolumn{5}{l}{* Results based on our own runs. $^\dagger$ Prompt-tuning baselines.}
    \end{tabular}
    \caption{DG results on ImageNet-DG and single-source DG on ImageNet and its variants. The `source' row indicates the source domain(s) used for fine-tuning. Best results are in bold. Most significant  improvements by incorporating GuiDG are underlined.}
    \label{tab3}
\end{table*}

\textbf{Implementation details.} 
We use CLIP ViT-B/16 \cite{radford2021learning} on all experiments. The temperature in \cref{eq3}, \cref{eq10} are set to 0.01 as in original CLIP. The query transformer $L_q$ in CMAttn is a linear layer of shape ($\mathrm{d_f}, \mathrm{d_f}$) where $\mathrm{d_f}$ is dimension of the feature representations in CLIP. $L_k$ is a linear layer of shape ($C,1$) that transforms prompt embedding of $C$ classes to 1. These linear layers in CMAttn introduce ~1M parameters, ensuring our method's parameter efficiency.
On all datasets, we follow leave-one-out paradigm to test one unseen target domain with others as source domains each time. On VLCS, PACS and TerraIncognita with fewer classes, all data  are used. On OfficeHome, DomainNet and ImageNet-DG we adopt few-shot fine-tuning protocol to  evaluate the few-shot generalization ability. To satisfy the data independence requirements in section Theoretical Formulation, we randomly split each domain to construct disjoint $D^S$ for phase 1 and $\widetilde{D}^S$ for phase 2. We adopt AdamW \cite{loshchilov2017decoupled} with learning rate 5e-6 on all tasks. The regularization loss $\mathcal{L}_r$ and $\alpha$ in \cref{opt2} are set according to original baseline methods, with more details in Appendix.  

\begin{table}[t]
    \centering
    \small
    \begin{tabular}{cc|cc|c}
    \hline
    \multicolumn{2}{c|}{Step 1} & \multicolumn{2}{c|}{Step 2} & results \\ \hline
    \begin{tabular}[c]{@{}c@{}}multiple\\ experts\end{tabular} & \begin{tabular}[c]{@{}c@{}}i.i.d\\ data\end{tabular} & \begin{tabular}[c]{@{}c@{}}no\\ weights\end{tabular} & \begin{tabular}[c]{@{}c@{}}learnable\\ weights\end{tabular} & \begin{tabular}[c]{@{}c@{}} Avg.\end{tabular} \\ \hline
        &  & \checkmark &  & 64.9 \\
    \checkmark &  & \checkmark &  & 67.0 \\
    \checkmark &  &  & \checkmark & 67.3 \\
        & \checkmark & \checkmark &  & 66.8  \\
    \checkmark & \checkmark & \checkmark &  & 67.6 \\
    \checkmark & \checkmark &  & \checkmark & \textbf{67.9} \\ \hline
    \end{tabular}
    \caption{Ablation study on ImageNet-DG. }
    \label{tab4}
\end{table}

\begin{figure}[t]
    \centering
    \includegraphics[width=\linewidth]{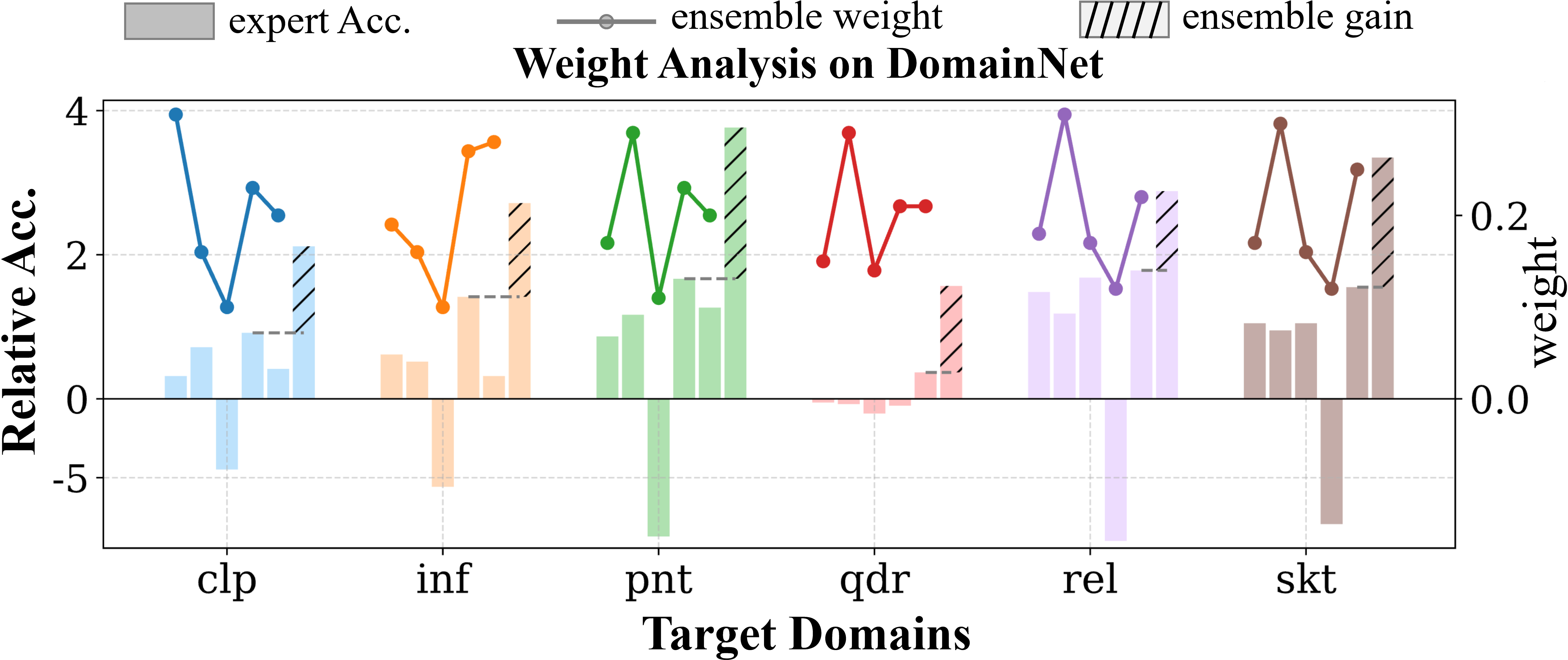}
    \caption{Bars and lines are relative accuracies (average accuracy subtracted) and weights of domain experts.  The rightmost bar in  each group shows the gains by prompt ensemble.}
    \label{fig4}
\end{figure}

\begin{figure*}[!t]
    \centering
    \subfloat[OH-Step 1]{\includegraphics[width=0.225\textwidth]{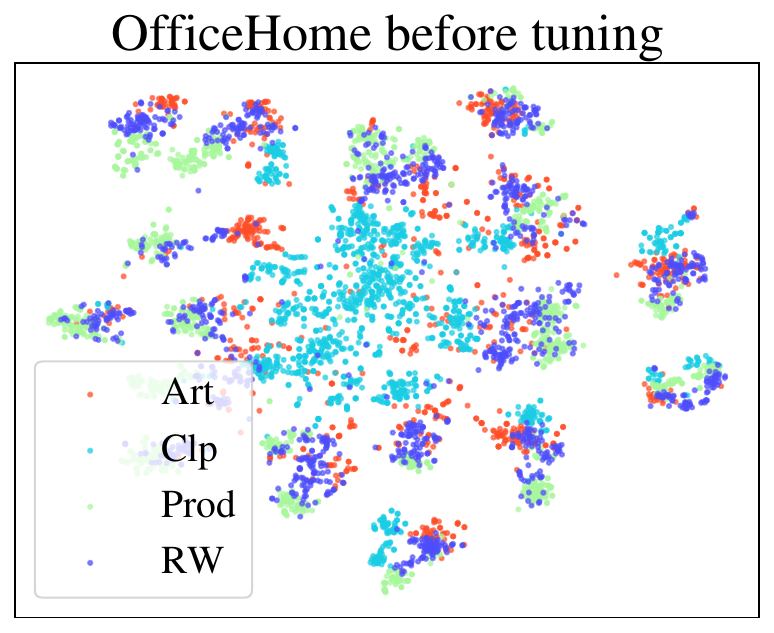}\label{fig3a}}
    \hfil
    \subfloat[OH-Step 2]{\includegraphics[width=0.225\textwidth]{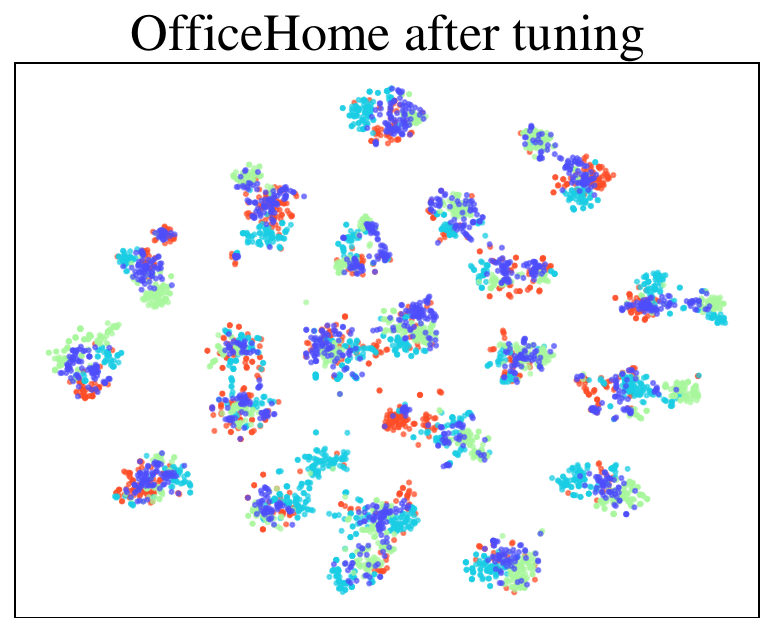}\label{fig3b}}
    \hfil
    \subfloat[DN-Step 1]{\includegraphics[width=0.225\textwidth]{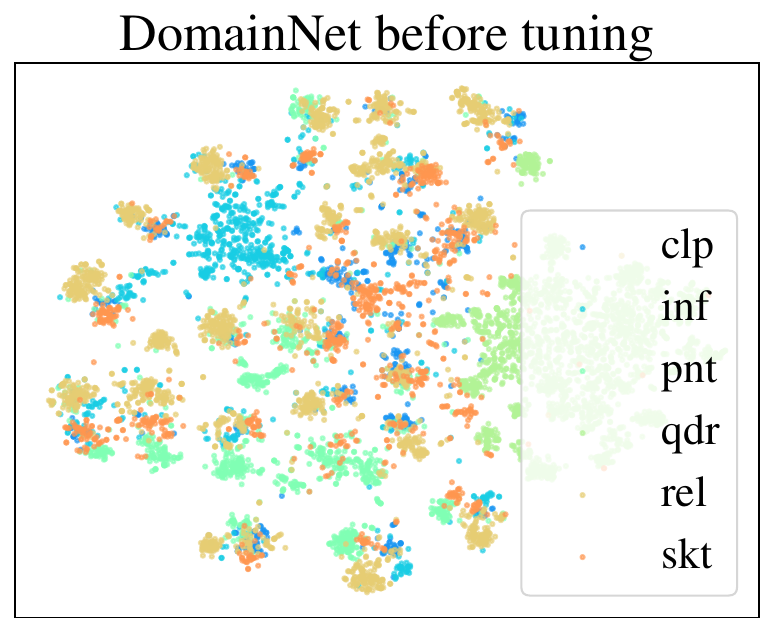}\label{fig3c}}
    \hfil
    \subfloat[DN-Step 2]{\includegraphics[width=0.225\textwidth]{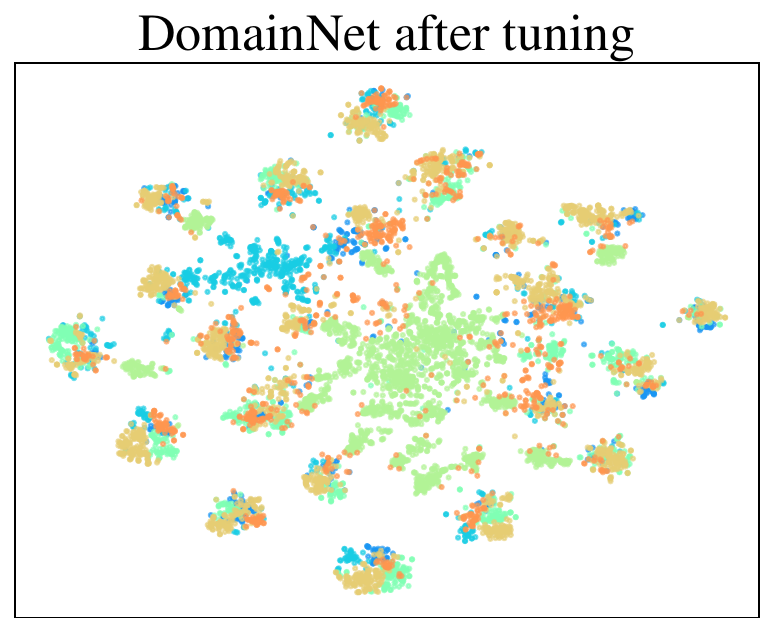}\label{fig3d}}
    \caption{Vision features of source and target data before and after fine-tuning. (a),(b) Results obtained on domain `Art' of OfficeHome. (c),(d) Results obtained on domain `clp' of DomainNet.}
    \label{fig3}
\end{figure*}

\begin{figure}[t]
    \centering
    \subfloat[Results on Office-Home]{\includegraphics[width=0.23\textwidth]{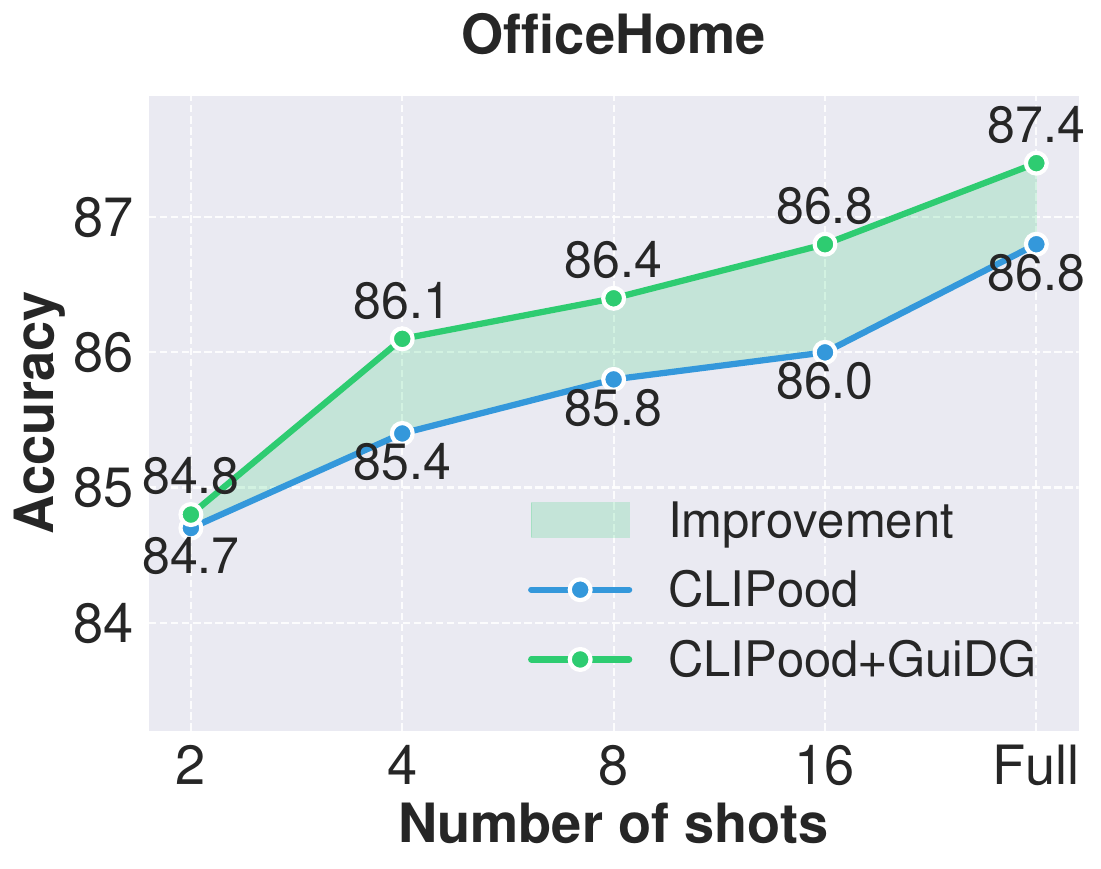}\label{fig5a}}
    \hfil
    \subfloat[Results on DomainNet]{\includegraphics[width=0.23\textwidth]{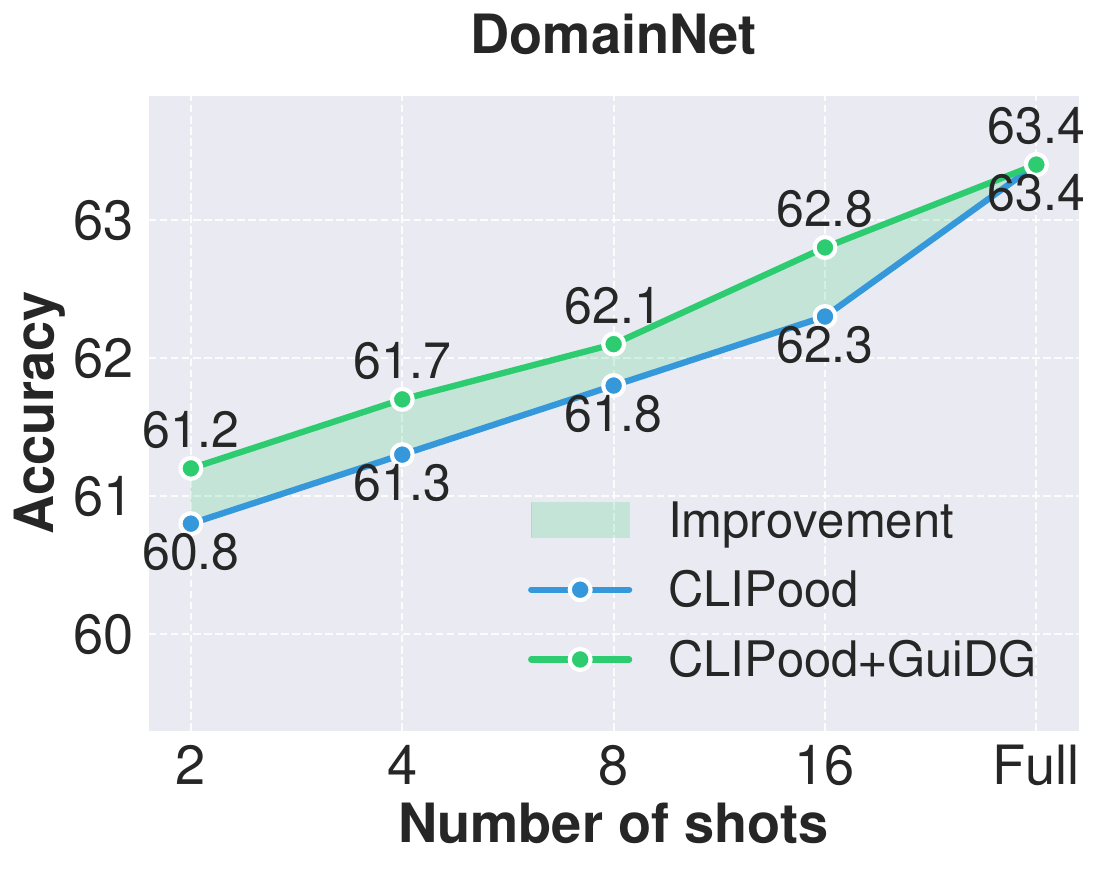}\label{fig5b}}
    \caption{Few-shot results, averaged over all target domains.}
    \label{fig5}
\end{figure}

\subsection{Main Results}
\label{sec:main_results}
\textbf{Standard benchmarks.} 
\cref{tab1} and \cref{tab2} include results on the five standard DG benchmarks in DomainBed~\cite{gulrajani2020search}. On OfficeHome and DomainNet we additionally provide few-shot results and comparisons. On all tasks, we adopt the prompt-tuning procedure in \cite{coop} for training domain experts with $m=16$. For regularization losses in \cref{opt2}, we experiment with the entropy loss in UEO \cite{ueo}, the Margin Metric Softmax loss in CLIPood \cite{clipood}, and ERM \cite{gulrajani2020search} (without regularization). To achieve competitive results, we incorporate WiSE-FT (WF) \cite{wiseft} when implementing ERM and UEO. The Beta Moving Average in CLIPood achieves similar effects. We provide domain-wise DG results in \cref{tab1} and average over all domains  in \cref{tab2}. We can observe that GuiDG provides steady improvements over the baseline methods on all tasks, achieving new state-of-the-art. Specifically, in few-shot settings, the significant reduce in fine-tuning data harms the generalization effects of fine-tuning methods. GuiDG mitigates such performance drop adaptive expert integration.

\textbf{Evaluation on ImageNet.} 
\cref{tab3} presents DG results on ImageNet and its variants. On ImageNet-DG, the source domains are ImageNet (train split), ImageNet-S and ImageNet-V2, and the target domains are ImageNet-A, ImageNet (evaluation split) and ImageNet-R. We ensure that in each task, the label space among source and target domains are identical. We can observe significant boosts brought by GuiDG on all baseline methods. The 8-shot performance with GuiDG even surpasses 16-shot performance without GuiDG, supporting the efficacy and data-efficiency of leveraging multiple dedicated domain models.
We also compare with  prompt-tuning methods~\cite{khattak2023self,yang2023towards}. The results show the superiority of our method on both ID (domain I) and OOD generalization tasks.
To evaluate GuiDG without domain labels, we perform single-source DG on ImageNet by randomly splitting the training set of ImageNet into 4 pseudo-domains. In such case, the domain experts tend to be homogeneous, but GuiDG still performs better than competing baselines. More discussions are in Appendix.

\subsection{Analytical Experiments}
\textbf{Ablation study.} 
\cref{tab4} presents ablation study of GuiDG. We  evaluate the effectiveness of key designs in Step 1 and 2 on baseline method 16-shot ERM (WF). The results indicate each component contributes positively. The most significant performance drop emerges if only one prompt is trained instead of multiple domain experts (`multiple experts'). The results support our design to `divide and conquer' the large source domain by smaller expert models. CMAttn further guides the fine-tuning process in Step 2 by combining appropriate domain experts (learnable weights), achieving better generalization than simple averaging (no weights). By utilizing i.i.d data between Step 1 and 2, the requirements of our theory are satisfied and improvements are observed.

\textbf{Weight analysis.} Proper ensemble of domain experts in GuiDG ensures minimum  generalization risks. \cref{fig4} investigates the compatibility between domain expert accuracies and their assigned weights, revealing the following insights. (1) CMAttn is generalizable to unknown domains. The assigned weights are  reasonable and compatible with the target performance of experts without accessing target data. 
(2) The ensemble process always provides positive gains. As shown by the rightmost bars in each group, the performances by integrating multiple domain experts consistently \textit{surpass} the best individual domain expert in the group (indicated by the shadowed parts in bars). 
(3) The weights take effect by reducing negative influences of experts. While it is hard to precisely match weights with every expert on unseen domains, CMAttn correctly assigns the lowest weights to the worst-performing experts in all cases to eliminate their drawbacks.

\begin{table}[t]
    \centering
    \small
    \begin{tabular}{l|ccc}
    \hline
    Dataset & CMAttn & Experts & Ratio \\\midrule
    OfficeHome (4) & 1.050M & 0.025M & 0.7\% \\
    DomainNet (6) & 1.052M & 0.041M & 0.7\% \\\bottomrule
    \end{tabular}
    \caption{Parameter analysis of GuiDG. In the parentheses are the number of domains. Ratio refers to the percentage of additional parameters by integrating GuiDG.}
    \label{param}
\end{table}

\textbf{Feature visualization.} We conduct t-SNE  visualization \cite{van2008visualizing} on vision features before and after the domain-expert-guided fine-tuning. As shown in  \cref{fig3}, before fine-tuning the features are chaotically distributed and cannot form compact locality structures, a merit that good classification models possess \cite{li2019locality,li2022source}. 
After our fine-tuning process, the features are more distinguishable and form distinct class-wise feature groups. We observe that on the previously unseen domains, the model can still extract discriminative features. 
On the most challenging domain `qdr' (\cref{fig3d}), its features after tuning better align with other domains compared to \cref{fig3c}.

\textbf{Few-shot performances.} To evaluate GuiDG with less  fine-tuning data, we compare the performance of CLIPood and CLIPood+GuiDG under the setting of 2-, 4-, 8-, 16-shots and full data fine-tuning in \cref{fig5}. The performances drop significantly with less training data, but we can still observe steady performance gains across all few-shot settings. 

\textbf{Parameter analysis.} \cref{param} analyzes additional parameters introduced by incorporating GuiDG. As the number of source domains increase, the additional parameters maintain at $\sim$1M in total. Such additional 1M parameters by incorporating GuiDG account for less than 1\% of all tunable parameters in current baselines, which is reasonable.

\section{Conclusion}
This work investigates domain generalization of VLMs. Current methods  train a universal model on all source domains for generalization, which is inevitably limited by the trade-off between model specificity and generalization ability. To address this, we show that ensemble of multiple smaller source expert models brings lower target risks while maintaining source specificity. Therefore, we design a domain-expert-guided DG framework that first learns prompt experts on source domains to encompass source knowledge. Secondly, a Cross-Modal Attention module is introduced to guide the tuning of VLMs with learnable weights. Experiments on standard DG benchmarks and a newly-proposed ImageNet-DG subset demonstrate the efficacy and efficiency of GuiDG.

\section{Acknowledgements}
This work was supported in part by the National Natural Science Foundation of China under Grant 62572102, 52441801, and in part by the Fundamental Research Funds for the Central Universities (UESTC) under Grant ZYGX2024Z008.
\bibliography{main}









\newpage
\maketitle
\section{Appendix}

\subsection{Discussion on Remark 3}
\label{sec0}
As established in Remark 3, under specific conditions, the ensemble model achieves a tighter upper bound on generalization risks compared to a universal model, i.e., $\mathrm{Upp}(\mathcal{E}^\prime(\widetilde{f}),\delta/3) < \mathrm{Upp}(\mathcal{E}^\prime(\widehat{f}),\delta)$. In GuiDG, we train expert functions $\widehat{f}_i$ with carefully constructed hypothesis spaces $\mathcal{H}_i$. Here, we demonstrate how our design fulfills the condition required in Remark 3, i.e., $\sum_{i=1}^{d} \pi_i^\prime\sqrt{2d_i}/\sqrt{\pi_i}\leq c(\delta)\sqrt{d_0}$, thereby achieving a tighter upper bound on generalization risks.

The upper bound of ensemble risks $\mathrm{Upp}(\mathcal{E}^\prime(\widetilde{f}),\delta)$ depends on VC-dimensions $\widetilde{d}$ and $d_i,i=1,\ldots,d$, while that of a universal model depends on $d_0$, the VC-dimension of the function space mapping the entire source domain to the target space. Denote the number of parameters in network space $\mathcal{H}$ as $n(\mathcal{H})$, we have $d_0=n(\mathcal{H})\log\{n(\mathcal{H})\},\widetilde{d}=n(\widetilde{\mathcal{H}})\log\{n(\widetilde{\mathcal{H}})\},d_i=n(\mathcal{H}_i)\log\{n(\mathcal{H}_i)\}$.
The original source domain inputs possess high dimensionality (e.g., image inputs). The ensemble module (e.g., CMAttn module), instead, only needs to incorporate several domain-specific outputs. Such modules possess much less tunable parameters compared to the function space that directly maps from the source domain, yielding $n(\widetilde{\mathcal{H}})\ll n(\mathcal{H})$ and $\widetilde{d}\ll d_0$.

By partitioning the whole source dataset into source sub-domains, we can  simplify the hypothesis space for each sub-task. In GuiDG, we adopt prompt-tuning to learn parameter-efficient domain experts for each sub-task. Assuming $\pi_i>\overline{\epsilon}>0$, application of the Cauchy inequality yields:
\begin{gather*}
\sum_{i=1}^d\pi_i^\prime\sqrt{2d_i}/\sqrt{\pi_i} \leq \sqrt{\sum_{i=1}^d\left(\pi_i^\prime\right)^2/\pi_i} \sqrt{\sum_{i=1}^d 2 d_i} \\
<\sqrt{\sum_{i=1}^d1/\pi_i}\sqrt{\sum_{i=1}^d2d_i}.
\end{gather*}

Let $\overline{c}_\pi=\sqrt{2\sum_{i=1}^d1/\pi_i}$. The assumption in Corollary 3.2 holds when:
\begin{gather*}
\overline{c}_\pi\sqrt{\sum_{i=1}^dd_i}\leq c(\delta)\sqrt{d_0} .
\end{gather*}

For a given source domain, $\overline{c}_\pi$ is constant and approximates $d$ when $\pi_i$ values are similar. Furthermore, since we typically consider tasks at a fixed probability level $1-\delta$, the terms $(1/d_0)\log(1/\delta)$ and $(1/d_i)\log(3/\delta)$ minimally impact $c(\delta)$ for sufficiently large $n$. We can establish that $c(\delta)>1-\widetilde{\epsilon}>0$. Thus, the condition:
\begin{gather*}
\sum_{i=1}^d d_i\leq \{c(\delta)/\overline{c}_\pi\}^2d_0,
\end{gather*}
is readily achievable by learning domain experts for each sub-task.

As a typical example, when $\pi_i=\pi_i^\prime$ and $c(\delta)\approx 1$, the condition simplifies to $2\sum_{i=1}^d d_i \leq d_0$. Let $n_i = n(\mathcal{H}_i)$, $2\sum_{i=1}^{d} n_i \le n(\mathcal{H})$, we have: 
\begin{align*}
    2\sum_{i=1}^d d_i & = 2 \sum_{i=1}^d n_i \log n_i \le 2 \sum_{i=1}^d n_i \log n(\mathcal{H}) \\
    & \le n(\mathcal{H}) \log n(\mathcal{H}) = d_0.
\end{align*}

We further provide an illustrative toy example to help understand our proposed bounds. We generate synthetic data nonlinearly from $f(x) = \text{sgn}(x)(3|\cos(x)| + x^2/2 + 3)$ with Gaussian noise. 
Fully-connected layers with structure $1 \to h \to h \to 1$ are then trained to fit the data. 
Baseline models include $h = h_1$ hidden unit. Our method splits and fits the data with 2 separate experts, each containing $h = 40$ hidden units. Their outputs are aggregated via a network with 3 hidden units. 
We test $h_1=$ 60,80,100 with 40 repeats, each with 200 training and 5000 test samples. Table \ref{toy} shows baseline risk $R_B$, our method's risk $R_O$, excess risks $E_B,E_O$, their ratio $R = E_B / E_O$, and theoretical ratio bound $r$. Both $R$ and $r$ grow as $h_1$ increases, with $r$ growing faster, validating the bound. $R_O$ is consistently lower than $R_B$ while requiring less parameters, indicating the efficacy of GuiDG.

\begin{table}[!h]
    \centering
    \caption{Experimental results on the toy example.}
    \label{toy}
    \resizebox{0.9\columnwidth}{!}{
    \setlength{\tabcolsep}{8pt}
    \begin{tabular}{c|cccccc}
    \toprule
    $h_1$ & $R_B$ & $R_O$ & $E_B$ & $E_O$ & $R$ & $r$ \\\midrule
    $60$ & $0.689$ & $0.599$ & $0.246$ & $0.220$ & $1.118$ & $1.120$ \\
    $80$ & $0.670$ & $0.599$ & $0.263$ & $0.220$ & $1.195$ & $1.482$ \\
    $100$ & $0.659$ & $0.599$ & $0.301$ & $0.220$ & $1.368$ & $1.843$\\\bottomrule
    \end{tabular}}  
    \vspace{-8pt}
\end{table}

\subsection{Proof of Theorem 1}
\label{sec1}
Before proving Theorem 1, we introduce the following lemma \cite{vapnik2013nature}.
\begin{lemma}\label{lemma: generalb}
    Assume hypothesis space $\mathcal{H}$ and $\mathcal{H}_i$ have VC-dimension $d_0$ and $d_i$ respectively. There exists constant $C>0$, such that for any $\delta\in(0,1)$ with probability at least $1-\delta$ following inequality hold:
    \begin{gather*}
        \underset{h\in\mathcal{H}}{\sup}\mathcal{E}(h)-\overset{d}{\underset{i=1}{\sum}}\pi_i\widehat{\mathcal{E}}_i(f)\leq C\sqrt{\dfrac{d_0\log(n)+\log(1/\delta)}{n}}\,,\\
        \underset{h\in\mathcal{H}_i}{\sup}\mathcal{E}_i(h)-\widehat{\mathcal{E}}_i(h)\leq C\sqrt{\dfrac{d_i\log(n_i)+\log(1/\delta)}{n_i}}\,.
    \end{gather*}
\end{lemma}
\begin{proof}[Proof of Theorem 1]
    Define the risks on each source domain data in Step 2 as $\widetilde{\mathcal{E}}_i(h)=\sum_{j=1}^{\widetilde{n}_i^S} \frac{1}{\widetilde{n}_i^S} \mathcal{L}(\widetilde{y}_j^i,h(\widetilde{x}_j^i))$, and recall the definition that $\widehat{\mathcal{E}}_i(h)=\sum_{j=1}^{n_i^S} \frac{1}{n_i^S}\mathcal{L}(y_j^i,h(x_j^i))$.
    As $\widetilde{f}=\mathcal{A}( \widehat{f}_1,\ldots,\widehat{f}_d;\widetilde{D}^S )$ is trained on additional dataset $\widetilde{D}^S$, following lemma \ref{lemma: generalb}, with probability at least $1-\delta$ where $\delta\in(0,1)$, we have
    \begin{equation*}
        \mathcal{E}_i(\widetilde{f})\leq\widetilde{\mathcal{E}}_i(\widetilde{f})+C\sqrt{\dfrac{\widetilde{d}\log(m)+\log(1/\delta)}{m}}\,.
    \end{equation*}
    Algorithm $\mathcal{A}$ aggregates $\widetilde{f}$  using $\widetilde{D}^S$ and ensures that for each source data point $(\widetilde{x}_j^i,\widetilde{y}_j^i)$, $\widetilde{f}$ behaves no worse than any $\widehat{f}_i,i=1,\ldots,d$. Therefore, we can derive 
    \begin{equation*}
        \widetilde{\mathcal{E}}_i(\widetilde{f})\leq\widetilde{\mathcal{E}}_i(\widehat{f}_i)\,.
    \end{equation*}
    The risk of $\widetilde{f}$ on $P^\prime$ can be divided. That is, with probability at least $1-d\delta$, we have
    \begin{align}
        \mathcal{E}^\prime(\widetilde{f})=&\overset{d}{\underset{i=1}{\sum}}\pi_i^\prime\mathcal{E}_i(\widetilde{f})\notag\\
        \leq&\overset{d}{\underset{i=1}{\sum}}\pi_i^\prime\widetilde{\mathcal{E}}_i(\widetilde{f})+C\sqrt{\dfrac{\widetilde{d}\log(m)+\log(1/\delta)}{m}}\notag\\
        \leq&\overset{d}{\underset{i=1}{\sum}}\pi_i^\prime\widetilde{\mathcal{E}}_i(\widehat{f}_i)+C\sqrt{\dfrac{\widetilde{d}\log(m)+\log(1/\delta)}{m}}\,.\label{ineq: divided1}
    \end{align}
    Note that $\widehat{f}_i$ is independent of $\widetilde{D}_i^S$. Using hoeffding inequality we can derive
    \begin{equation*}
        \mathbb{P}\left( \widetilde{\mathcal{E}}_i(\widehat{f}_i)-\mathcal{E}_i(\widehat{f}_i)\leq t \right)\geq 1-\exp\left( -2\widetilde{n}_i^St^2/c_L \right)\,.
    \end{equation*}
    Let $t=\sqrt{\dfrac{c_L\log(1/\delta)}{2\widetilde{n}_i^S}}$, we have with probability at least $1-d\delta$, $\forall i=1,\ldots,d$
    \begin{equation}
        \widetilde{\mathcal{E}}_i(\widehat{f}_i)\leq\mathcal{E}_i(\widehat{f}_i)+\sqrt{\dfrac{c_L\log(1/\delta)}{2\widetilde{n}_i^S}}\,.\label{ineq: divided2}
    \end{equation}
    Use lemma \ref{lemma: generalb} again, with probability at least $1-d\delta$, we have for $\forall i=1,\ldots,d$
    \begin{align}
        \mathcal{E}_i(\widehat{f}_i)\leq\widehat{\mathcal{E}}_i(\widehat{f}_i)+C\sqrt{\dfrac{d_i\log(n_i^S)+\log(1/\delta)}{n_i^S}}\,.\label{ineq: divided3}
    \end{align}
    Combine inequality (\ref{ineq: divided1}), (\ref{ineq: divided2}) and (\ref{ineq: divided3}) together we have
    \begin{align}
        \mathcal{E}^\prime(\widetilde{f})\leq&\overset{d}{\underset{i=1}{\sum}}\pi_i^\prime\widetilde{\mathcal{E}}_i(\widehat{f}_i)+C\sqrt{\dfrac{\widetilde{d}\log(m)+\log(1/\delta)}{m}}\notag\\
        \leq&\overset{d}{\underset{i=1}{\sum}}\pi_i^\prime\mathcal{E}_i(\widehat{f}_i)+\left( \overset{d}{\underset{i=1}{\sum}}\pi_i^\prime/\sqrt{\pi_i} \right)\sqrt{\dfrac{c_L\log(1/\delta)}{2m}}\notag\\
        &+C\sqrt{\dfrac{\widetilde{d}\log(m)+\log(1/\delta)}{m}}\notag\\
        \leq&\overset{d}{\underset{i=1}{\sum}}\pi_i^\prime\widehat{\mathcal{E}}_i(\widehat{f}_i)+\left( \overset{d}{\underset{i=1}{\sum}}\pi_i^\prime/\sqrt{\pi_i} \right)\sqrt{\dfrac{c_L\log(1/\delta)}{2m}}\notag\\
        &+C\sqrt{\dfrac{\widetilde{d}\log(m)+\log(1/\delta)}{m}}\notag\\
        &+C \overset{d}{\underset{i=1}{\sum}}\pi_i^\prime\sqrt{\dfrac{d_i\log(n_i^S)+\log(1/\delta)}{n_i^S}}\,.\label{form: boundforAgg1}
    \end{align}
    For target risk on $\widehat{f}$, utilize lemma \ref{lemma: generalb} again, with probability at least $1-d\delta$ we have
    \begin{align}
        \mathcal{E}^\prime(\widehat{f})-\overset{d}{\underset{i=1}\sum}\pi_i^\prime\widehat{\mathcal{E}}_i(\widehat{f})\leq C\sqrt{\dfrac{d_0\log(N)+\log(1/\delta)}{N}}\,.
        \label{form: boundforClassical1}
    \end{align}
\end{proof}

\subsection{Proof of Corollary 2}
\label{sec2}
In this section we prove Corollary 2.
\begin{proof}[Proof of Corollary 2]
    As $n_i^S=\pi_in$, we can reformulate $\mathrm{Upp}(\mathcal{E}^\prime(\widetilde{f}),\delta/3)$ as 
    \begin{align*}
        &\mathrm{Upp}(\mathcal{E}^\prime(\widetilde{f}),\delta/3)\\
        =&C \overset{d}{\underset{i=1}{\sum}}\pi_i^\prime\sqrt{\dfrac{d_i\log(n_i^S)+\log(3/\delta)}{n_i^S}}+\varepsilon\\
        =&C \overset{d}{\underset{i=1}{\sum}}\dfrac{\pi_i^\prime\sqrt{2d_i}}{\sqrt{\pi_i}}\sqrt{\dfrac{\log(n_i^S)+(1/d_i)\log(3/\delta)}{N}}+\varepsilon\\
        \leq&C \overset{d}{\underset{i=1}{\sum}}\dfrac{\pi_i^\prime\sqrt{2d_i}}{c(\delta)\sqrt{\pi_i}}\sqrt{\dfrac{\log(N)+(1/d_0)\log(1/\delta)}{N}}+\varepsilon\\
        \leq&C\sqrt{\dfrac{d_0\log(N)+\log(1/\delta)}{N}}+\varepsilon\\
        =&\mathrm{Upp}(\mathcal{E}^\prime(\widehat{f}),\delta)+\varepsilon\,,
    \end{align*}
    where the first inequality is due to the definition of $c(\delta)$.
\end{proof}

\begin{table*}[t]
    \centering
    \caption{Statistics of ImageNet-DG.}
    \resizebox{0.9\linewidth}{!}{
    \begin{tabular}{ccccccc}
    \toprule
    \multicolumn{1}{c}{\multirow{2}{*}{target domain}} & \multicolumn{1}{c}{\multirow{2}{*}{class count}} & \multicolumn{1}{c}{\multirow{2}{*}{ target samples}} & \multicolumn{3}{c}{ sample number of source domains} & \multicolumn{1}{c}{\multirow{2}{*}{total}} \\ \cmidrule(lr){4-6}
    \multicolumn{1}{c}{} & \multicolumn{1}{c}{} & \multicolumn{1}{c}{} & ImageNet (train) & ImageNet-S & ImageNet-V2 & \multicolumn{1}{c}{} \\ \cmidrule(r){1-7} 
    ImageNet-A & 200 & 7500 & 259906 & 10169 & 2000 & 272075 \\
    ImageNet (eval) & 1000 & 50000 & 1281167 & 50889 & 10000 & 1342056 \\
    ImageNet-R & 200 & 30000 & 258951 & 10152 & 2000 & 271103 \\ \bottomrule
    \end{tabular}}
\label{app_tab1}
\end{table*}

\begin{figure}[t]
    \centering
    \includegraphics[width=\columnwidth]{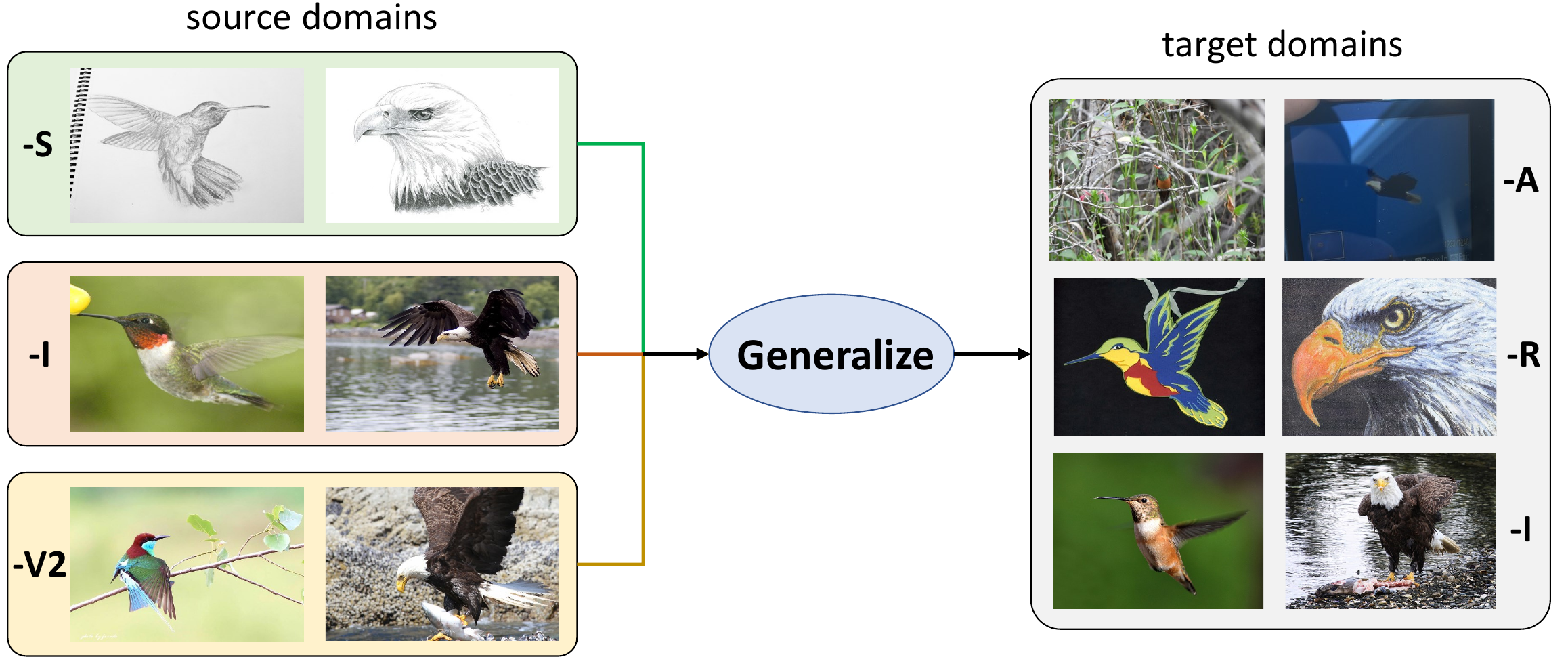}
    \caption{Example pictures of ImageNet-DG. The presented images are from classes `hummingbird' and `bald eagle'.}
    \label{app_fig1}
  \end{figure}

\subsection{Details on ImageNet-DG}
\label{sec3}
We construct ImageNet-DG from ImageNet \cite{deng2009imagenet} and its four variants (ImageNet-A \cite{hendrycks2021natural}, ImageNet-R \cite{hendrycks2021many}, ImageNet-S \cite{wang2019learning} and ImageNet-V2 \cite{recht2019imagenet}). ImageNet-DG includes 3 target domains to generalize to, i.e.,  ImageNet-A, ImageNet (evaluation split) and ImageNet-R. For all 3 tasks, the source domains are ImageNet (train split), ImageNet-S and ImageNet-V2. However, ImageNet-A and ImageNet-R only includes 200 classes sampled from the 1000 classes in ImageNet. In the problem setting of domain generalization, the label distribution $P_{Y|X}$ of source and target domains should be the same. Therefore, we only select source samples that \textit{share categories with the target domain} for each task. The statistics of the resultant ImageNet-DG dataset are in \cref{app_tab1}.
\cref{app_fig1} presents example pictures in ImageNet-DG. The model needs to incorporate knowledge from various source domains (e.g., sketch and natural style) and generalize to unknown domains. The target domains include natural adversarial (-A) samples that are hard to recognize even for humans, art-style (-R) pictures and natural pictures (-I). The large amount of samples provides robust and comprehensive assessment of model generalization ability.

\subsection{Detailed Experiment Results}
\label{sec4}
\begin{table*}[t]
    \caption{Detailed results on TerraIncognita, VLCS and PACS. Best results are in bold.}
    \centering
    \resizebox{\linewidth}{!}{
    \begin{tabular}{l|ccccc|ccccc|ccccc}
    \toprule
     & \multicolumn{5}{|c}{TerraIncognita} & \multicolumn{5}{|c}{VLCS} & \multicolumn{5}{|c}{PACS} \\ \cmidrule(l){2-16} 
    \multirow{-2}{*}{Methods} & L100 & L38 & L43 & L46 & Avg. & C & L & S & V & Avg. & A & C & P & S & Avg. \\ \cmidrule(r){1-16}
    CLIP \cite{radford2021learning} & 50.8 & 23.4 & 32.2 & 28.8 & 33.8 & \textbf{100.0} & 67.4 & 73.5 & 86.1 & 81.8 & 97.6 & 98.9 & \textbf{100.0} & 88.2 & 96.2 \\
    ERM (WF) & 57.1 & 54.8 & 48.5 & 43.6 & 51.0 & \textbf{100.0} & 66.1 & 76.8 & 90.1 & 83.3 & 97.8 & 99.1 & \textbf{100.0} & 89.9 & 96.7 \\
    \rowcolor[HTML]{DDEBF7} 
    ERM (WF)+ GuiDG & 59.2 & 56.0 & 50.1 & 46.2 & 52.9 & \textbf{100.0} & 68.0 & 77.4 & 89.9 & 83.8 & 98.5 & 99.4 & \textbf{100.0} & \textbf{92.5} & \textbf{97.6} \\
    UEO (WF) \cite{ueo} & 58.7 & 57.5 & 47.6 & 42.2 & 51.5 & \textbf{100.0} & 60.1 & 76.2 & 88.9 & 81.3 & 98.1 & 98.9 & \textbf{100.0} & 90.7 & 96.9 \\
    \rowcolor[HTML]{DDEBF7} 
    UEO (WF)+ GuiDG & 57.7 & \textbf{60.0} & 48.2 & 43.4 & 52.3 & \textbf{100.0} & 66.9 & \textbf{79.6} & 90.1 & 84.2 & 98.9 & \textbf{99.6} & \textbf{100.0} & 91.8 & \textbf{97.6} \\
    CLIPood \cite{clipood} & 73.1 & 58.4 & \textbf{57.7} & 50.9 & 60.0 & 98.9 & 68.2 & 78.8 & \textbf{90.8} & 84.2 & \textbf{99.0} & \textbf{99.6} & \textbf{100.0} & 90.7 & 97.3 \\
    \rowcolor[HTML]{DDEBF7} 
    CLIPood+ GuiDG & \textbf{74.7} & 59.8 & 56.8 & \textbf{52.3} & \textbf{60.9} & 99.3 & \textbf{69.7} & 77.9 & 90.1 & \textbf{84.3} & 98.3 & \textbf{99.6} & \textbf{100.0} & 91.1 & 97.3 \\ \bottomrule
    \end{tabular}}
\label{app_tab2}
\end{table*}

We extend Table 2 in main paper by presenting domain-wise results on TerraIncognita, VLCS and PACS. Results are in \cref{app_tab2}. Each column  represents one generalization task. The column names are the target domains. Specifically,  in TerraIncognita the columns names are locations of camera, in VLCS (V-VOC2007, L-LabelMe, C-Caltech101, S-SUN09) are names of sub-datasets, and in PACS (P-photo, A-art painting, C-cartoon, S-sketch) are art styles. We can observe that incorporating GuiDG always brings positive overall gains, and that on each dataset our GuiDG achieves the best results. On all tasks, the results are significantly higher than zero-shot CLIP. On TerraIncognita, GuiDG enhances CLIP with abundant domain-specific knowledge, while on VLCS and PACS, GuiDG preserves the pretrained knowledge in CLIP. Therefore, GuiDG achieves consistent superiority on various generalization scenarios.

\subsection{Implementation Details of GuiDG}
\label{sec6}

\begin{algorithm*}
    \caption{Two-step training algorithm for GuiDG.}
    \label{alg:main}
    \begin{algorithmic}[1]
    \Procedure{Training domain experts.}{}
        \State \textbf{Input:} source dataset $D^S$, number of source domains $d$, CLIP encoders $E_v$ and $E_t$.
    \For{$i$ in $[1,2,...,d]$}
        \State Obtain domain-specific data $D_i^S$ from $D^S$ (which is readily available in multi-source DG tasks, and is randomly divided in single-source DG tasks).
        \While{not converged}
            \State Sample data points $(x_j^i, y_j^i)$ from $D_i^S$.
            \State Compute training loss $\mathcal{L}_p^i$ in \cref{sup_eq1}.
            \State Update learnable prompt embeddings $\mathbf{p}_i$ by minimizing $\mathcal{L}_p^i$.
        \EndWhile
    \EndFor
    \State \Return Trained domain experts $\{\mathbf{p}_i\}_{i=1}^d$.
    \EndProcedure
    
    \Procedure{Fine-tuning CLIP with dedicated prompt guidance.}{}
    \State \textbf{Input:} Trained domain experts $\{\mathbf{p}_i\}_{i=1}^d$, additional dataset $\widetilde{D}^S$, off-the-shelf fine-tuning regularization term $\mathcal{L}_r$, regularization weight $\alpha$, CLIP encoders $E_v$ and $E_t$.
    \While{not converged}
        \State Sample data points $(x_j, y_j)$ from $\widetilde{D}^S$.
        \State Compute domain weights as in \cref{sup_eq2}.
        \State Compute training loss $\mathcal{L}_f$ in \cref{sup_eq3}.
        \State Update parameters in $E_v$ and CMAttn by minimizing $\mathcal{L}_f$.
    \EndWhile
    \State \Return Fine-tuned CLIP vision encoder parameters $\theta_{E_v}$, trained CMAttn with parameters  $\{\theta_{L_q},\theta_{L_k}\}$.
    \EndProcedure
    \end{algorithmic}
\end{algorithm*}

Detailed training algorithm for GuiDG are in \cref{alg:main}. Recall the training losses for domain experts as \cref{sup_eq1}:
\begin{align}
    \mathcal{L}^i_p= - \sum_{j=1}^{n_i^S} \sum_{c=1}^{C} \bm{1}[y_j^i=c] \cdot \log P_c(\hat{y}\mid x^i_j,t^i_{1:C}),
\label{sup_eq1}
\end{align}
where $P_c(y|x)$ computes the probability that output $y$ belongs to class $c$.

During the fine-tuning of CLIP, we first compute domain importance by CMAttn:
\begin{align}
    \mathbf{w}(x) = \mathrm{Softmax}(\mathrm{cos} \left\langle q(x), [k_1,k_2,\cdots,k_d] \right\rangle ),
\label{sup_eq2}
\end{align}
where $q(\cdot)$ is the query transformation in CMAttn, and $k_i$ are keys transformed from domain experts. We then compute overall loss by weighting the loss for each domain:
\begin{align}
    \mathcal{L}_{f} = \sum_{i=1}^{d} \sum_{j=1}^{\widetilde{n}^S_i} w_i(x_j^i) \cdot \left( - \sum_{c=1}^{C} \bm{1}[y_j^i=c] \cdot \log P_c(\hat{y}|x_j^i,t^i_{1:C}) \right).
    \label{sup_eq3}
\end{align} 
We implement \cref{alg:main} and conduct all experiments with PyTorch on one NVIDIA RTX 4090 GPU.

Below we  introduce the off-the-shelf regularization techniques $\mathcal{L}_r$ mentioned in Algorithm \ref{alg:main}, line 14. In this paper we mainly build our method upon three DG baselines: UEO \cite{ueo}, CLIPood \cite{clipood} and WiSE-FT (WS) \cite{wiseft}. 

\textbf{UEO} introduces universal entropy minimization as regularization: $\mathcal{L}=\sum_{x} \tilde{w}(x) \mathcal{H}(p(x)) - \mathcal{H}(\bar{p})$, where $\mathcal{H}(p(x))=-\sum_{c=1}^{C}p_c(x) \log p_c(x)$ denotes the Shannon entropy of $p(x)$, $\tilde{w}(x)=\frac{1}{B}$ where $B$ is batch size, $p_c(x)$ is classification probability for class $c$. The regularization weight is set to $\alpha=0.1$ as in \cite{ueo}.

\textbf{CLIPood} introduces beta moving average (BMA) for weight averaging, and margin metric softmax (MMS) loss. BMA  maintains a moving average
model $\theta^{\text{BMA}}$ and at each time step $t$, and the current model $\theta^t$ is added into $\theta^{\text{BMA}}_t$ to update the moving average:
\begin{align}
    \theta^{\text{BMA}}_t = \frac{\sum_{k=0}^{t-1} \alpha_k}{\sum_{k=0}^{t} \alpha_k} \cdot \theta^{\text{BMA}}_{t-1} + \frac{\alpha_t}{\sum_{k=0}^{t} \alpha_k}  \cdot \theta_t,
\end{align}
where
\begin{align}
    \alpha_t = \mathrm{Beta}(\beta, \beta) \left( \frac{t+0.5}{T+1} \right),
\end{align}
where $\mathrm{Beta}(\beta, \beta)$ is beta distribution and $\beta=0.5$ is hyperparameter. The MMS loss is given by:
\begin{align}
    \mathcal{L} = -\log \frac{\exp (S(I_x, T_y) / \tau)}{\sum_{c=1}^{C} \exp (S(I_x, T_c) + \lambda \cdot D(T_y, T_c)) / \tau},
\end{align}
where $D(T_y, T_c) = 1-S(T_y, T_c)$, $S(\cdot,\cdot)$ is similarity score between feature representations.

\textbf{WiSE-FT} is another weight averaging technique. Consider pretrained model weight $\beta_0$ and weights after fine-tuning $\beta_1$, the weight-space ensemble is given by:
\begin{align}
    \beta_{ens} = (1-\alpha) \cdot \beta_0 + \alpha \cdot \beta_1,
\end{align}
where $\alpha=0.5$ is hyperparameter.

For more detailed description please refer to the original papers \cite{ueo,clipood,wiseft}. Our method is compatible with more generalization techniques, which could potentially further improve the performance.

\end{document}